\documentclass[letterpaper]{article} 
\usepackage{aaai25}  
\usepackage{times}  
\usepackage{helvet}  
\usepackage{courier}  
\usepackage[hyphens]{url}  
\usepackage{graphicx} 
\usepackage{natbib}  
\usepackage{caption} 
\usepackage{algorithm}
\usepackage{algorithmic}
\usepackage{newfloat}
\usepackage{listings}
\DeclareCaptionStyle{ruled}{labelfont=normalfont,labelsep=colon,strut=off} 
\floatstyle{ruled}
\newfloat{listing}{tb}{lst}{}
\floatname{listing}{Listing}
\usepackage{amsmath}
\usepackage{amsthm}
\usepackage{amsfonts}
\usepackage{mathtools}
\usepackage{amssymb}
\usepackage{natbib}
\newtheorem{proposition}{Proposition}

\newtheorem{theorem}[proposition]{Theorem}

\def\eqdef{\stackrel{\text{def}}{=}}

\newcommand{\Ocal}{\mathcal{O}}

\usepackage{mathtools}

\DeclarePairedDelimiter\br{(}{)}
\DeclarePairedDelimiter\brs{[}{]}
\DeclarePairedDelimiter\brc{\{}{\}}
\DeclarePairedDelimiter\norm{\lVert}{\rVert}
\DeclarePairedDelimiter\inner{\langle}{\rangle}
\usepackage{subcaption}
\newcommand\Ex[1]{\mathbb{E}\brs*{#1}}

\newcommand{\R}{\mathbb{R}}

\usepackage{thmtools}
\usepackage{thm-restate}
\usepackage{tcolorbox}
\usepackage{booktabs}
\usepackage{multirow}
\newtheorem{theorem-rst}[proposition]{Theorem}
\newtheorem{lemma-rst}[proposition]{Lemma}
\newtheorem{proposition-rst}[proposition]{Proposition}
\newtheorem{assumption-rst}[proposition]{Assumption}
\newtheorem{claim-rst}[proposition]{Claim}
\newtheorem{corollary-rst}[proposition]{Corollary}

\title{Latent Reward: LLM-Empowered Credit Assignment \\
in Episodic Reinforcement Learning}
\author{
    Yun Qu\equalcontrib, Yuhang Jiang\equalcontrib, Boyuan Wang, Yixiu Mao, Cheems Wang, Chang Liu, Xiangyang Ji\thanks{Corresponding Author}\\
}
\affiliations{
    Tsinghua University\\
    \{qy22, jiangyh19, wangby22, myx21\}@mails.tsinghua.edu.cn,\\ cheemswang@mail.tsinghua.edu.cn, 
    liuchang2022@tsinghua.edu.cn, xyji@tsinghua.edu.cn

}

\begin{document}

\maketitle

%
\begin{abstract}

Reinforcement learning~(RL) often encounters delayed and sparse feedback in real-world applications, even with only episodic rewards. 
Previous approaches have made some progress in reward redistribution for credit assignment but still face challenges, including training difficulties due to redundancy and ambiguous attributions stemming from overlooking the multifaceted nature of mission performance evaluation. 
Hopefully, Large Language Model (LLM) encompasses fruitful decision-making knowledge and provides a plausible tool for reward redistribution.
Even so, deploying LLM in this case is non-trivial due to the misalignment between linguistic knowledge and the symbolic form requirement, together with inherent randomness and hallucinations in inference. 
To tackle these issues, we introduce \textbf{LaRe}, a novel LLM-empowered symbolic-based decision-making framework, to improve credit assignment. 
Key to LaRe is the concept of the \textit{Latent Reward}, which works as a multi-dimensional performance evaluation, enabling more interpretable goal attainment from various perspectives and facilitating more effective reward redistribution. 
We examine that semantically generated code from LLM can bridge linguistic knowledge and symbolic latent rewards, as it is executable for symbolic objects.
Meanwhile, we design latent reward self-verification to increase the stability and reliability of LLM inference. 
Theoretically, reward-irrelevant redundancy elimination in the latent reward benefits RL performance from more accurate reward estimation.
Extensive experimental results witness that LaRe (i) achieves superior temporal credit assignment to SOTA methods, (ii) excels in allocating contributions among multiple agents, and (iii) outperforms policies trained with ground truth rewards for certain tasks.

\end{abstract}

\section{Introduction}

Episodic reinforcement learning is dedicated to solving problems of receiving only episodic rewards, a frequently encountered situation in real-world applications of RL, such as autonomous driving~\cite{kiran2021deep} and healthcare~\cite{zeng2022optimizing}.
Credit assignment~\cite{sutton2011horde, zhang2020learning}, which involves assessing the contributions of single-step decisions~\cite{ren2021learning}, is challenging in episodic RL due to delayed and sparse feedback.
Return decomposition~\cite{arjona2019rudder}, which estimates proxy rewards by using state-action pairs to redistribute episodic rewards, has emerged in the literature as a promising direction to remedy this issue.
Subsequent works often focus on model architectures~\cite{liu2019sequence, widrich2021modern} or human-designed regression principles~\cite{ren2021learning, lin2024episodic}, overlooking the training difficulty posed by redundant information.
\citet{zhang2024interpretable} attempted to address this redundancy by employing a causal approach to filter out reward-irrelevant features but still struggled with the lack of semantic interpretation.

A prominent observation in human problem-solving is that contribution assessments often encompass a range of qualitative and quantitative factors.
For instance, soccer players' performance is evaluated not only by goals scored but also by injury prevention and coordination. 
Similarly, the rewards designed in RL are commonly a combination of multiple factors~\cite{todorov2012mujoco, qu2023hokoff}.
Previous methods~\cite{arjona2019rudder, ren2021learning} mainly focus solely on the values of final returns without tapping into the multifaceted nature of performance evaluation, resulting in poor semantic interpretability and ambiguous credit assignment.
Recently, the demonstrated capabilities of pre-trained LLM~\citep{achiam2023gpt} suggest that
integrating its prior knowledge for improved credit assignment is a promising solution.
However, the misalignment between LLM's linguistic knowledge and the symbolic representations required for specific tasks poses significant challenges, while the inherent randomness and hallucinations in LLM inference further diminish its effectiveness~\cite{peng2023self,carta2023grounding}.

Motivated by the urgent demand of depicting multifaceted performance evaluation, we propose a key concept for credit assignment, termed \textit{Latent Reward}, where different dimensions capture various aspects of task performance while eliminating reward-irrelevant redundancy.
We then devise a framework \textbf{LaRe}, which (i) derives semantically interpretable latent rewards by incorporating task-related priors from LLM and (ii) utilizes them to enhance reward decomposition.
With the insight that semantically generated code can bridge linguistic knowledge in LLM and targets in symbolic form due to its executability for symbolic objects, LaRe presents a general paradigm for integrating LLM's prior knowledge into symbolic tasks.
Specifically, pre-trained LLM is instructed by standardized prompts to code encoding functions, which encode environment information into semantically interpretable latent rewards, eliminating the need for task-specific training.
To alleviate the randomness and hallucinations in LLM reasoning, LaRe designs a self-verification mechanism for stable and reliable latent reward derivation.

Our \textbf{main contributions} are summarized as follows:
\begin{enumerate}
    \item We propose the concept of Latent Reward with semantical interpretability and reveal the multifaceted nature of step-wise contributions by introducing it in the probabilistic model of episodic rewards, aligning with human preferences and reducing redundancy.
    \item We devise a latent reward-based framework, LaRe, to leverage LLM's task-related priors for more accurate and interpretable credit assignment, which paves a way for integrating LLM into symbolic-based decision-making.
    \item We demonstrate the superiority of our method both theoretically and empirically and validate LLM's effectiveness as a generalized information encoder for latent reward derivation in practical implementation.
\end{enumerate}

The surprising phenomenon that LaRe outperforms policies trained with ground truth dense rewards for certain tasks highlights the significance of the semantically interpretable latent rewards derived through LLM's reasoning capability.
Our work reveals that merely fitting the final reward value, which primarily reflects overall performance, may be insufficient for effective reward decomposition.
This suggests that RL can be further enhanced through multifaceted performance assessments informed by task-related priors.

\begin{figure*}[htbp]
    \centering
    \begin{subfigure}[b]{0.17\linewidth}
        \centering
        \includegraphics[width=\linewidth]{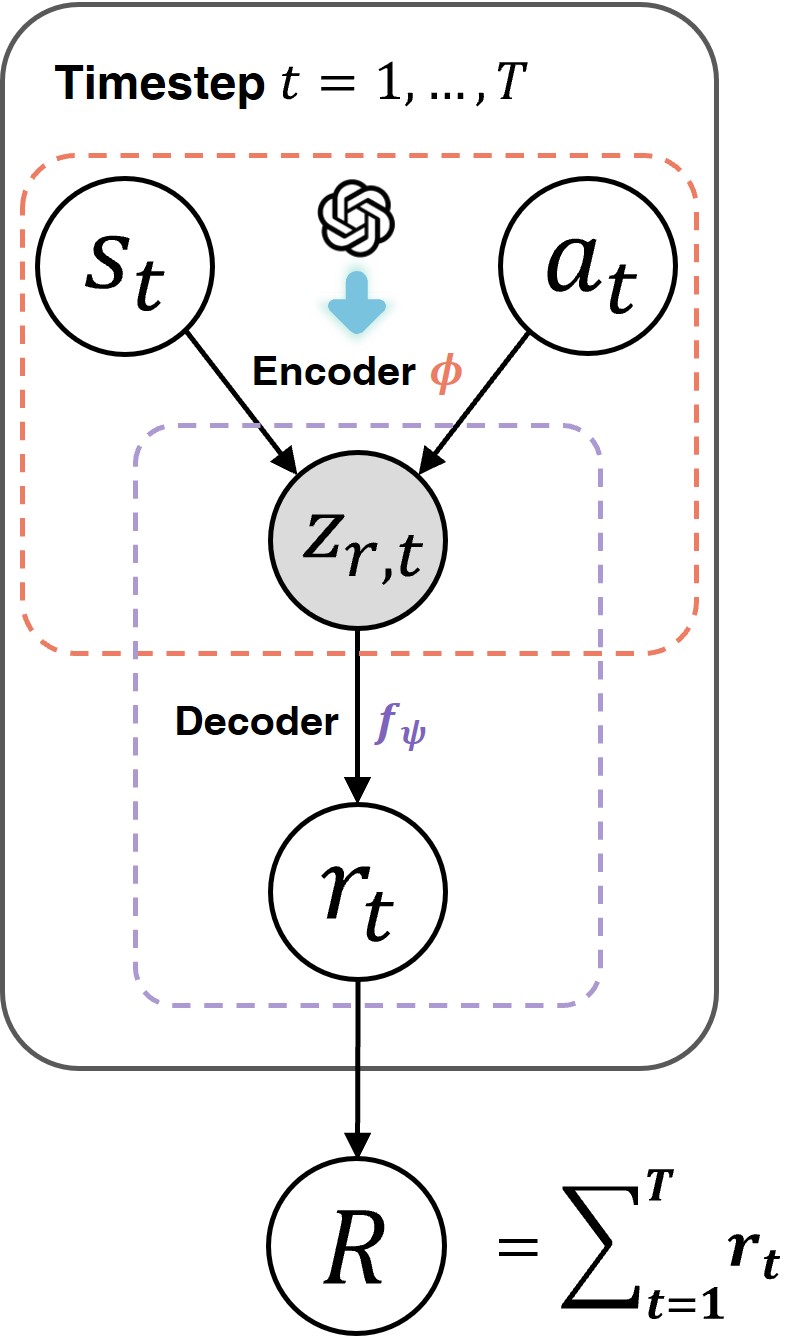}
        \caption{}
        \label{fig:return_pm}
    \end{subfigure}
    \hfill
    \begin{subfigure}[b]{0.81\linewidth}
        \centering
        \includegraphics[width=\linewidth]{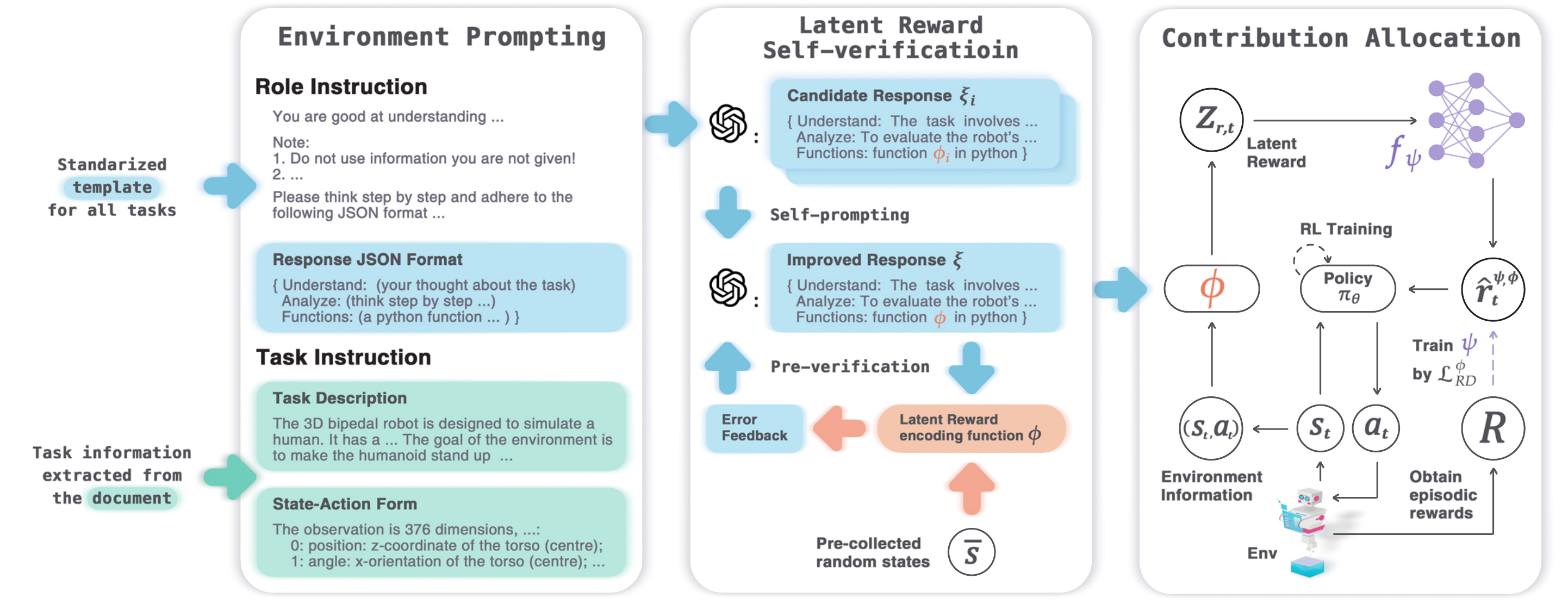}
        \caption{}
        \label{fig:framework}
    \end{subfigure}
    \caption{
    Overview of LaRe. (a) The probabilistic model of the episodic reward with the latent reward $z_{r, t}$ introduced as the implicit variable.
    (b) The LaRe framework consists of three main components: 
    (1) Environment Prompting: the task information is incorporated into a standardized prompt for LLM instructions (details are in Appendix A).
    (2) Latent Reward Self-verification: during the self-prompting phase, LLM generates $n$ candidate responses $\{\xi_i\}_{i=1}^n$ and synthesizes an improved response $\xi$. In the pre-verification phase, the executability of the function $\phi$ is verified with pre-collected random states $\bar{s}$;
    (3) Contribution Allocation: latent rewards $z_{r, t}$ are derived by $\phi$ and used to estimate proxy rewards via the reward decoder model $f_\psi$.}
    \label{fig:both}
\end{figure*}

\section{Related Works}
\subsection{Reward Redistribution}
Reward redistribution seeks to transform episodic rewards into immediate and dense proxy rewards $\hat{r}_t$, re-assigning credit for each state-action pair~\cite{ren2021learning, zhang2024interpretable}.
Some previous methods focus on reward shaping~\cite{ng1999policy, hu2020learning} and intrinsic reward design~\cite{pathak2017curiosity, zheng2021episodic}.
Return decomposition has emerged as a promising approach for tackling scenarios with severely delayed rewards.
RUDDER~\cite{arjona2019rudder} analyzes the return-equivalent condition for invariant optimal policy and proposes return decomposition via a regression task. 
Subsequent works build on it by aligning demonstration sequences~\cite{patil2020align}, using sequence modeling~\cite{liu2019sequence}, or Hopfield networks~\cite{widrich2021modern}. 
\citet{ren2021learning} propose randomized return decomposition to bridge between return decomposition~\cite{efroni2021reinforcement} and uniform reward redistribution~\cite{gangwani2020learning}. 
Other redistribution principles have been adopted in recent works, such as causal treatment~\cite{zhang2024interpretable} and randomly cutting sub-trajectories~\cite{lin2024episodic}. 
Recently, some methods have used attention-based approaches to decompose returns across time and agents in multi-agent settings~\citep{she2022agent, xiao2022agent, chen2023stas}.
Despite significant progress, previous studies have neglected redundant reward-irrelevant features and the multifaceted nature of mission performance evaluation, which impede training and cause ambiguous attributions.
While \citet{zhang2024interpretable} have acknowledged this issue to some extent, they focus solely on extracting reward-related state elements. 
In contrast, we propose the latent reward as a semantically interpretable multi-dimensional performance measurement and achieve reward-irrelevant redundancy elimination with task-related priors.

\subsection{LLM-Empowered Decision Making}

The remarkable capabilities of LLMs, as demonstrated across various downstream tasks~\citep{touvron2023llama, brown2020language}, underscores their potential as a promising solution for decision-making~\citep{wang2023survey}. 
Some works focus on high-level control by employing LLMs as planners with predefined skills or APIs, which have proven highly successful~\citep{liang2023code, yao2022react, shinn2023reflexion,  zhu2023ghost, wang2023voyager,zhang2024large}.
However, when directly applied to low-level control without predefined skills, the misalignment between LLMs' linguistic knowledge and the symbolic states and actions required for specific tasks poses a significant challenge~\citep{peng2023self,qu2024choices}. 
Some works address this issue by constructing text-based environments but at the cost of considerable manual effort~\citep{du2023guiding, carta2023grounding}. 
Recently, LLMs have been integrated with RL to enhance low-level control~\cite{cao2024survey}.
Some approaches fine-tune LLMs as policies~\citep{carta2023grounding, shi2024LaMo} or use LLM for history compression~\citep{paischer2022history}.
Other studies~\cite{zhang2023bootstrap,su2023subgoal,shukla2023lgts} focus on goal-conditioned RL with LLMs as subgoal selectors, but these often require predefined skills or subgoals. 
We seek to leverage LLMs as tools to enhance RL, aligning with LLM-based reward design methods~\citep{kwon2023reward, song2023self, wang2024llm}.
However, our method ensures a more reliable and optimized use of LLM priors by strategically designing for improved response quality and integrating them into latent rewards during the training process for optimization rather than relying on unreliable direct use.

\section{Preliminary}

The environments in reinforcement learning are generally formulated by a Markov Decision Process~(MDP; \citet{bellman1966dynamic}), which can be defined as a tuple $\mathcal{M}=\langle \mathcal{S},\mathcal{A},\gamma,P,r \rangle$, where $\mathcal{S}$ and $\mathcal{A}$ denote the state space and action space with cardinalities $\norm{\mathcal{S}}$ and $\norm{\mathcal{A}}$, respectively.
$\gamma \in [0, 1)$ is the discount factor. 
$P(s'|s, a)$ represents the environment's state transition distribution, and $r(s,a)$ denotes the reward function. 
The goal of reinforcement learning is to find an optimal policy $\pi: \mathcal{S} \rightarrow \mathcal{A}$ that maximizes the expected cumulative rewards with the initial state distribution $\eta$ and episode length $T$, which is expressed as $J(\pi)=\mathbb{E}\left[\sum_{t=1}^T \gamma^t r\left(s_t,\pi(s_t)\right)|s_0\sim\eta, s_{t+1}\sim P\left(\cdot|s_t,\pi(s_t)\right)\right]$.

Real-world scenarios often pose challenges such as delayed and sparse feedback~\cite{ke2018sparse,  han2022off}. 
An extreme case is episodic RL, where only a non-zero reward $R(\tau)$ at the end of each trajectory $\tau$ is received~\cite{ren2021learning}. The goal of episodic reinforcement learning is to maximize the expected episodic rewards, i.e., $J_{ep}(\pi)=\mathbb{E}\left[R(\tau)|s_0\sim\eta, a_t\sim\pi(\cdot|s_t), \tau=	\left \langle s_0, a_0, s_1,...,s_T\right\rangle \right]$.
A common assumption in episodic RL is the existence of a sum-form decomposition of the episodic rewards, i.e., $R(\tau) = \sum_{t=1}^T r(s_t,a_t)$~\cite{zhang2024interpretable}.

\section{Latent Reward}

This section elaborates on LaRe's motivation and implementation.
We explain the rationale behind the \textit{Latent Reward} and analyze the underlying probabilistic model. 
We propose a framework \textit{LaRe} that leverages LLM's reasoning and generalization capabilities  while addressing the challenges of its application  to incorporate task-related prior for reliably deriving the latent reward.
We theoretically prove that by reducing reward-irrelevant redundancy, the latent reward enhances reward modeling and improves RL performance.

\subsection{Motivation}
In human endeavors, individual contributions are typically assessed from multiple angles for a comprehensive evaluation.
However, current research on episodic credit assignment often focuses solely on regressing the final reward values~\cite{arjona2019rudder, efroni2021reinforcement}, overlooking that rewards are derived from the evaluation of various implicit factors, such as costs and efficiency.
Inspired by the intrinsic need to evaluate task performance from multiple perspectives, we propose the concept of the \textit{Latent Reward}.
Conceptually, the different dimensions of latent reward capture various aspects of task performance. 

Formally, the reward $r$ is a projection of the latent reward $z_r$ from a space $\mathcal{D}$ with cardinality $\norm{\mathcal{D}}$ onto the real number field $\mathbb{R}$.
A function $f: \mathcal{D} \rightarrow \mathbb{R}$ should exist such that each reward in the reward codomain has at least one latent reward encoding.
With the introduction of the latent reward, as illustrated in Figure \ref{fig:return_pm}, we construct a new probabilistic model of the episodic reward, revealing the multifaceted nature of the step-wise contribution, which better serves RL training.
We have,
\begin{align}
&\quad p(R \vert s_{1:T}, a_{1:T}) \nonumber 
= \int p(R, r_{1:T}, z_{r, 1:T}| s_{1:T}, a_{1:T}) \, \mathrm{d}z \, \mathrm{d}r \nonumber \\
&=\int \left[ \prod_{t=1}^T \underbrace{p(r_t | z_{r, t})}_{decoder f} \underbrace{p(z_{r, t} | s_t, a_t )}_{encoder \phi} \right] p(R | r_{1:T}) \, \mathrm{d}z \, \mathrm{d}r
\end{align}
where the $\phi: \mathcal{S \times A} \rightarrow \mathcal{D}$ is the function deriving the latent reward from environment information.
Intuitively, the latent reward's multiple dimensions are obtained by compressing environmental information based on prior knowledge, thus acting as an information bottleneck~\cite{tishby2000information} tailored to the task objectives.

Compared to directly estimating step-wise rewards from raw states, the latent reward offers significant advantages in interpretability, as each dimension reflects a specific aspect of task performance. 
Additionally, in episodic RL, where only the return of an episode provides weak signals, directly modeling rewards can be challenging. 
Learning from latent rewards better aligns with task objectives and simplifies network training by reducing reward-irrelevant redundancy.

A naive approach is to obtain the latent reward via an information bottleneck method, which suffers from limited linguistic interpretability and high computational costs due to separate encoder training for each task.
In contrast, LLM's pre-training has captured more compact representations in the form of tokens, facilitating better cross-task generalization. 
Therefore, leveraging LLM's prior knowledge enables more efficient extraction of interpretable and multifaceted task performance metrics, the latent reward, from the redundant environmental information.

\begin{algorithm}[t]
\caption{LaRe}
\label{alg:algorithm}
\textbf{Input}: LLM $\mathcal{M}$, task information $task$, role instruction $role$, candidate responses number $n$, pre-collected random state-action pairs $\Bar{s}$, max episodes $\mathcal{N}^{max}$\\
\textbf{Output}: policy network~$\pi_\theta$, reward decoder model~$f_\psi$

\begin{algorithmic}[1] 
\STATE Initialize the policy network parameter $\theta$, the reward decoder model parameter $\psi$, and the replay buffer $\mathcal{B}$.
\STATE Obtain response $\xi$ by executing Eq.~\eqref{init_gen} and Eq.~\eqref{sum_gen}.
\STATE Repeat Eq.~\eqref{feedback} until obtaining an executable $\phi$.
\FOR{$episode=1$ \textbf{to} $\mathcal{N}^{max}$}
\STATE Sample a trajectory $\tau$ using current policy.
\STATE $\mathcal{B}\gets\mathcal{B}\cup\{\tau\}$. Sample a batch $B=\{\tau_i\}_{i=1}^{|B|}$ from $\mathcal{B}$.
\STATE Estimate latent reward enhanced return decomposition loss $\mathcal{L}^\phi_{RD}(\psi)$ with Eq.~\eqref{eq:our_rd} and update reward decoder model $f_\psi$
\begin{equation}
    \psi \leftarrow \psi - \alpha\mathcal{r}_\psi\mathcal{L}^\psi_{RD}(\psi)
\end{equation}
\STATE Perform policy optimization using any RL algorithm with predicted proxy rewards $\hat{r}^{\psi, \phi} = f_\psi(\phi(s, a))$.
\ENDFOR
\end{algorithmic}
\end{algorithm}
\subsection{Framework}
Leveraging LLM's prior knowledge and reasoning capabilities to derive latent rewards for credit assignment presents three main challenges:
(1) instructing LLM to derive latent rewards for various tasks with minimal information and effort, 
(2) addressing the linguistic-symbolic misalignment while mitigating randomness and hallucinations in LLM inference to derive symbolic latent rewards reliably, and 
(3) applying latent rewards to enhance contribution allocation at each timestep.
This section introduces three specifically designed components in the proposed LaRe, as demonstrated in Fig.~\ref{fig:framework} and Algorithm~\ref{alg:algorithm}:

\textbf{Environment Prompting.} 
To instruct LLM, we design standardized prompts easily transferable across environments, which consist of a templated role instruction~($role$) and specific task instruction~($task$), as shown in Fig.~\ref{fig:framework}. 
The role instruction is consistent across tasks and guides LLM to think in a predefined manner: understand the task and state $\rightarrow$ identify reward-related factors $\rightarrow$ generate the latent reward encoding function. 
Only the necessary task description and state forms for a specific task are required, which can be easily extracted from the task document.
The task description mainly includes the environment profile and task objective.
The state forms detail the meanings of dimensions in the state space.
Our design significantly reduces the burden of labor-intensive prompt engineering across tasks.

\textbf{Latent Reward Self-verification.}
Since LLM's knowledge is encoded in language while underlying tasks are represented by symbolic states, this misalignment impedes LLM's direct application. 
To effectively integrate LLM, we propose generating the latent reward encoding function using LLM’s coding capabilities. 
The rationale is that semantically generated code can bridge the gap between linguistic knowledge and symbolic latent rewards, as its execution is symbolic and tailored to specific tasks, as previously confirmed~\cite{wang2024llm}. 
Given the inherent randomness and hallucinations in LLM inference, inspired by recent work~\cite{shinn2023reflexion, ma2023eureka}, we propose a latent reward LLM generation process with self-verification, which includes \textit{self-prompting} and \textit{pre-verification} to enhance stability and reliability.

In the self-prompting phase, LLM $\mathcal{M}$ firstly generates $n$ candidate responses, each including a code implementation of the latent reward encoding function: 
\begin{equation}
\xi_1, \xi_2, \ldots, \xi_n \leftarrow \mathcal{M}(task, role)
\label{init_gen}
\end{equation}
These candidate responses are then fed into the prompt, and LLM is prompted to summarize an improved response:
\begin{equation}
\xi \leftarrow \mathcal{M}(task, role, \xi_{1 \ldots n})
\label{sum_gen}
\end{equation}

Regarding pre-verification, leveraging the standardized response template, the latent reward encoding function $\phi$ can be easily extracted from the response $\xi$, which takes in a state-action pair $s,a$ and outputs a latent reward $z_{r}=\phi(s)=[z_r^1,\ldots, z_r^d]$. 
We then verify $\phi$ with pre-collected random state-action pairs $\Bar{s}$ and provide error feedback to LLM until $\phi$ is executable:
\begin{equation}
    err \leftarrow verify(\phi, \Bar{s});\ \xi \leftarrow \mathcal{M}(task, role, \xi_{1 \ldots n}, err)
\label{feedback}
\end{equation}

Self-verification significantly improves response quality by reducing randomness in identifying latent rewards and ensuring code executability. 
LLM's clear linguistic responses and transparent thought processes provide high interpretability, facilitating human evaluation and manual intervention. Empirical results demonstrate that our framework achieves satisfactory results without requiring multi-iteration evolutionary optimization~\cite{ma2023eureka}.

\begin{figure*}[htbp]
  \centering
    \includegraphics[width=\linewidth]{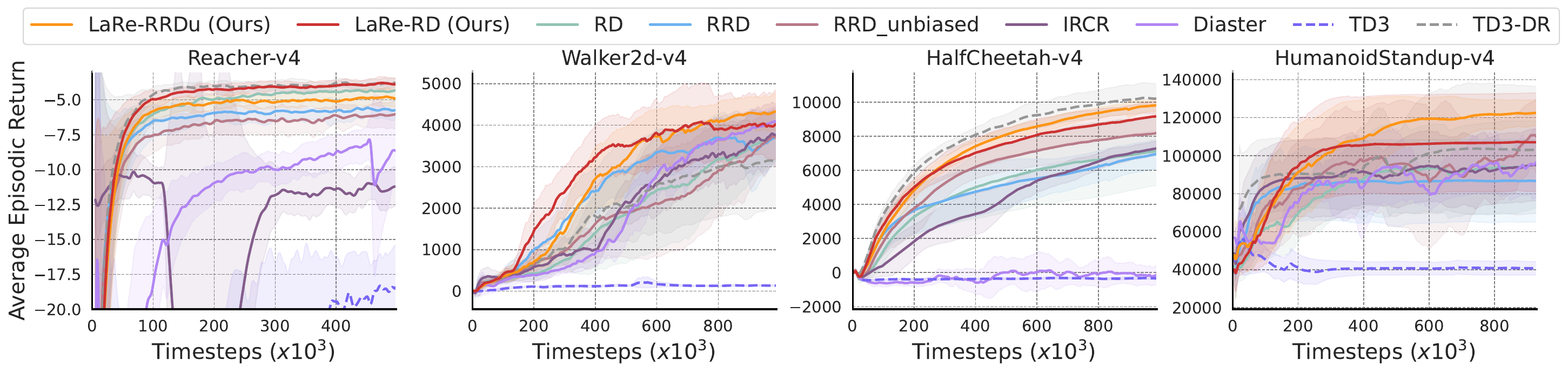}
    \caption{Average episode return for tasks with different state space dimensions in MuJoCo. Notably, \textbf{TD3-DR} is trained with dense rewards.
    }
    \label{fig:mujoco}
\end{figure*}

\textbf{Contribution Allocation.} Building on the latent reward encoding function, we adopt a latent reward enhanced return decomposition, implemented based on \citet{efroni2021reinforcement}.
Let $f_\psi$ be a neural network decoder parameterized by $\psi$.
The new objective of reward modeling can be formulated as:
\begin{equation}
\footnotesize
    \min_\psi \mathcal{L}^\phi_{RD}(\psi) = \mathbb{E}_{\tau\sim D} \left[\left(R(\tau)-\sum_{t=1}^T f_\psi(\phi(s_t, a_t))\right)^2\right]
    \label{eq:our_rd}
\end{equation}
Proxy rewards, $\hat{r}^{\psi, \phi} = f_\psi(\phi(s, a))$, derived from latent rewards, are incorporated into the RL training process.
Leveraging the enhanced temporal credit assignment enabled by the latent reward's multifaceted nature, these rewards improve RL training performance by alleviating the issue of delayed and sparse feedback.

Additionally, we empirically find that the latent reward enhances credit assignment among agents.
This well matches the intuition, as evaluating agents within a team is also a form of multifaceted credit assignment.
Consequently, our method provides a practical solution for episodic multi-agent RL, with reduced computational costs and improved performance, making it well-suited for real-world scenarios.

In implementations, we use GPT-4o from OpenAI API, with prompt details provided in Appendix A.
In practice, we have set the random variables deterministically for the sake of convenience, which is a common setting in previous works~\cite{arjona2019rudder}.

\subsection{Analysis}
LLM-empowered latent rewards retain semantic interpretability while reducing reward-irrelevant redundancy, which is theoretically proven to boost RL performance by learning a better reward model than the state-based methods.

Previous works commonly minimize the least squares error between the episodic rewards and the sum of predicted proxy rewards $\hat{r}(s_t, a_t)$ to learn reward models with raw states as inputs~\cite{ren2021learning}.
The surjective function $\phi(s, a): \mathcal{S}\times\mathcal{A}\rightarrow\mathcal{D}, \mathcal{\norm{D}}<\mathcal{\norm{S}}\mathcal{\norm{A}}$  reduces redundant, reward-irrelevant features from the state-action space.
Theoretically, built upon \citet{efroni2021reinforcement}, assuming access to a latent reward function $\phi$ that satisfies $\exists f^*, s.t., r=\hat{r}=f^*(\phi(s, a))$, we derive a more precise concentration bound for estimating $r$ and a tighter RL regret bound compared to the case without the latent reward. 
Please refer to Appendix B for the proof.

\begin{restatable}[\textbf{Tighter Concentration Bound of Reward}]{proposition-rst}{rewardConcentration}\label{proposition: concentration of reward}
Let $\lambda>0$ and $A^\phi_{k}\eqdef (H^\phi_k)^TH^\phi_k +\lambda I_{\mathcal{\norm{D}}}$. For any $\delta\in\br*{0,1}$, with probability greater than $1-\delta/10$ uniformly for all episode indexes $k\ge0$, it holds that 
\small
$$\norm{r-\hat{r}^\phi_k}_{A^\phi_{k}} \leq  \sqrt{\frac{1}{4}T \mathcal{\norm{D}}\log\br*{\frac{1+kT^2/\lambda}{\delta/10}}}+\sqrt{\lambda \mathcal{\norm{D}}} \eqdef l^\phi_k < l_k. $$
\end{restatable}

\begin{restatable}[\textbf{Tighter Regret Bound}]{proposition-rst}{OFULrlPer}\label{theorem: OFUL for RL with per trajectory feedback}
For any $\delta\in (0,1)$ and all episode numbers $K\in\mathbb{N}^+$, the regret of RL $\rho^\phi(K) \eqdef \sum_{k=1}^K\left(V^*-V^{\phi,\pi_k}\right)$
holds with probability greater than $1-\delta$ that,
\scriptsize
$$
\rho^\phi(K)\leq \Ocal\br*{T \mathcal{\norm{D}} \sqrt{ K} \log\br*{\frac{KT}{\delta}}} < \Ocal\br*{T \mathcal{\norm{S}\norm{A}} \sqrt{ K} \log\br*{\frac{KT}{\delta}}}.
$$
\end{restatable}

The concentration bound reflects the performance of the reward model by quantifying the distance between proxy rewards $\hat{r}^\phi_k$ and true rewards $r$, while the regret quantifies RL performance. Proposition~\ref{proposition: concentration of reward} and \ref{theorem: OFUL for RL with per trajectory feedback} show that these bounds are proportional to $\mathcal{\norm{D}}$, which are lower than the bound with raw state-action space. Overall, the latent reward improves reward function learning and boosts RL performance.

\section{Experiments}

We evaluate LaRe\footnote{Our code is available at \url{https://github.com/thu-rllab/LaRe}} on two widely used benchmarks in both single-agent and multi-agent settings: MuJoCo locomotion benchmark~\cite{todorov2012mujoco} and Multi-Agent Particle Environment (MPE)~\cite{lowe2017multi}.  
Additionally, we perform ablation studies and further analyses to validate LaRe's components and assess its properties.

\subsection{Experimental Setups}

For MuJoCo, we adopt four tasks from Gymnasium~\cite{towers_gymnasium_2023}.
For MPE, we employ six tasks from two scenarios, \textit{Cooperative-Navigation (CN)} and \textit{Predator-Prey (PP)}, featuring varying numbers of agents (6, 15, 30), which are based on \citet{chen2023stas} with minor modifications to provide individual rewards to each agent at every step. 
All tasks are episodic, with a single non-zero episodic reward, equivalent to the cumulative rewards.
Thus, multi-agent tasks require both temporal and inter-agent credit assignment.
Moreover, we evaluate LaRe in more complex scenarios from SMAC~\cite{samvelyan19smac} and \textbf{a newly designed task, \textit{Triangle Area}}, in Appendix D and E.

\begin{figure*}[htbp]
  \centering
    \includegraphics[width=0.82\linewidth]{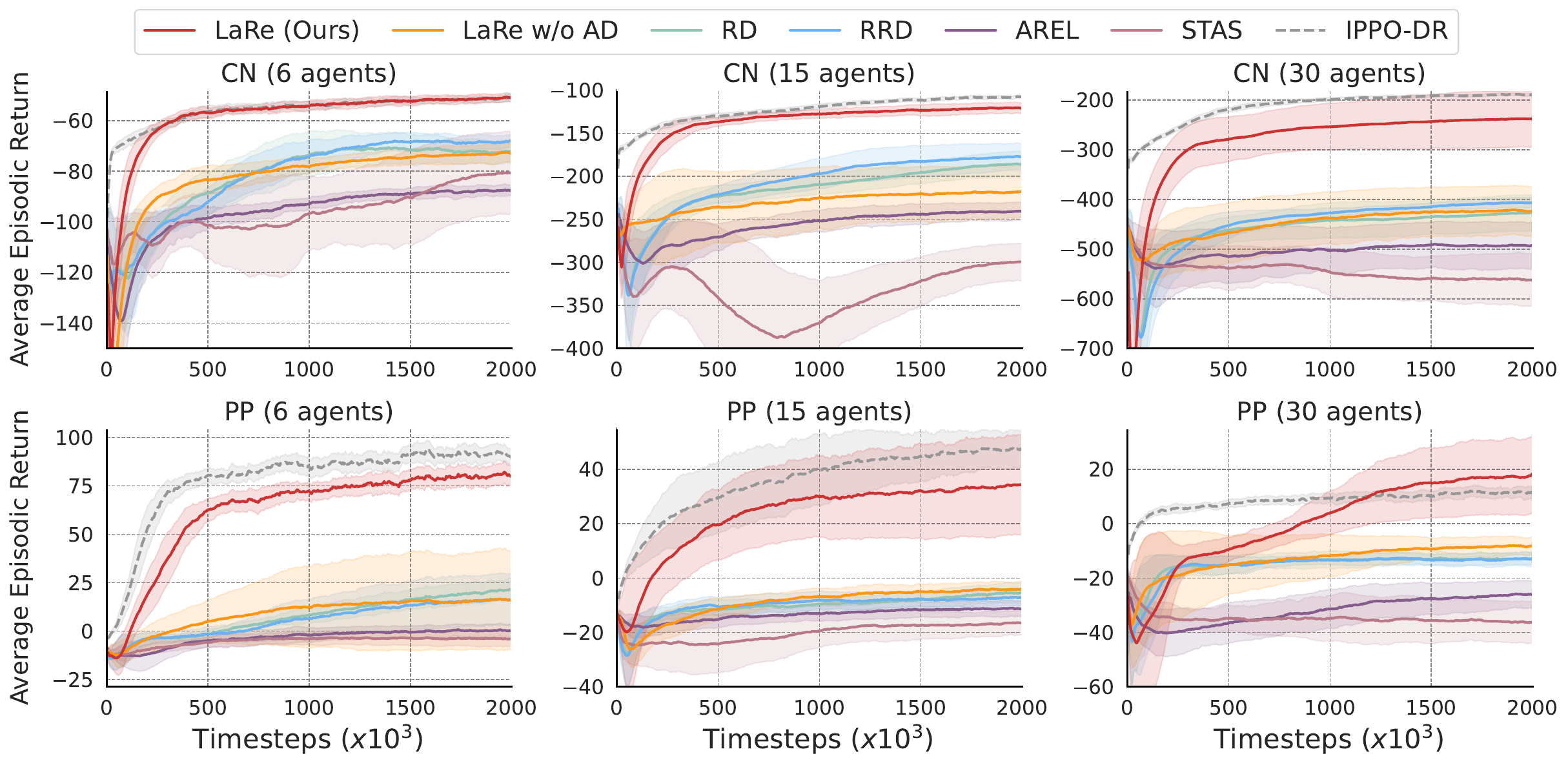}
    \caption{Average episode return for tasks with a varied number of agents in MPE. 
    Notably, \textbf{IPPO-DR} is trained with dense rewards and \textbf{LaRe w/o AD} represents LaRe without credit assignment among agents.}
    \label{fig:mpe}
\end{figure*}
We compare LaRe with SOTA return decomposition baseline algorithms: RD~\cite{efroni2021reinforcement}, IRCR~\cite{gangwani2020learning}, Diaster~\cite{lin2024episodic}, RRD and RRD\_unbiased~\cite{ren2021learning}, as well as those designed for multi-agent settings: AREL~\cite{xiao2022agent} and STAS~\cite{chen2023stas}. 
The introduction and implementation details of these baselines are provided in Appendix C.

LaRe is compatible with various RL algorithms, and we adopt TD3~\cite{fujimoto2018addressing} for single-agent and IPPO~\cite{yu2022surprising} for multi-agent as the base algorithm, consistent with prior works~\cite{ren2021learning, chen2023stas}.
Each algorithm runs on five random seeds, with the mean performance and standard deviation reported. 
Further details and results are available in the Appendix.

\subsection{The Superiority of LaRe}

\textbf{Single-Agent.}
To verify the compatibility of our method with various return decomposition algorithms, we implement two variants, \textbf{LaRe-RD} and \textbf{LaRe-RRDu}, based on RD and RRD-unbiased, respectively. 
As shown in Fig.~\ref{fig:mujoco}, the poor performance of TD3 and IRCR highlights the importance of assigning individual credits. 
Our method, LaRe, consistently outperforms SOTA baselines on MuJoCo tasks, demonstrating higher sample efficiency and better convergence. 
Both variants of LaRe surpass the corresponding baselines, highlighting the efficacy of semantically interpretable latent rewards in credit assignment. 
The effectiveness of LaRe in tasks with large state spaces significantly supports our analysis, underscoring the significance of redundancy elimination with task-related priors in the latent reward and explaining the poor performance of baselines.

\textbf{Multi-Agent.}
Fig.~\ref{fig:mpe} depicts comprehensive comparisons between LaRe and various baselines in MPE. 
LaRe is implemented based on RD and demonstrates superior performance across tasks with different numbers of agents compared to all SOTA baselines, confirming the efficacy of latent rewards in temporal credit assignment on multi-agent tasks.
We also include a variant \textbf{LaRe w/o AD} (without agent decomposition), where the proxy rewards of different agents are averaged at the same time step.
The significant performance drop highlights the necessity of credit assignment at the agent level and the effectiveness of LaRe in this regard.
We believe the semantically interpretable latent rewards account for this since assessing different agents' contributions is also intuitively a form of multifaceted credit assignment.
AREL and STAS perform relatively poorly, particularly as the number of agents increases, likely because reward-irrelevant items in the original state interfere significantly with attention-based credit assignment.

\textbf{Comparable with Dense Rewards.}
We include \textbf{TD3-DR} and \textbf{IPPO-DR} in MuJoCo and MPE, respectively, training with ground truth dense rewards.  
Remarkably, LaRe's performance is comparable to or even exceeds theirs despite not relying on manually designed rewards.
The reason is that while ground true rewards reflect agents' performance levels, overall stability is still affected by implicit factors like costs and efficiency, which are adequately captured by our proposed LLM-based latent reward.
This finding emphasizes leveraging task-related prior information for multifaceted performance evaluation can further enhance RL performance beyond merely relying on final reward values.

\subsection{Delving into Latent Rewards}
We conduct experiments to analyze the specific nature of the latent rewards and the reason for their superior performance.

\textbf{Semantic Analysis of Multifaceted Measurement.}
We analyze the LLM-generated latent reward functions and use \textit{HumanoidStandup-v4} as an instance. 
The task objective is to have the humanoid robot stand up and maintain balance by applying torques to hinges~\cite{towers_gymnasium_2023}. 
As shown in Fig.~\ref{fig:hetero}(b), LLM demonstrates a correct understanding of the task and derives latent rewards as interpretable performance measures across multiple dimensions, such as height and safe control, which align with the ground truth~(GT) reward function. 
Additionally, LLM considers stability, which better aligns with the task's objectives, further elucidating its superior performance compared to baselines with dense rewards. 
Further details can be found in Appendix A.

\begin{figure}[t]
  \centering
    \includegraphics[width=\linewidth]{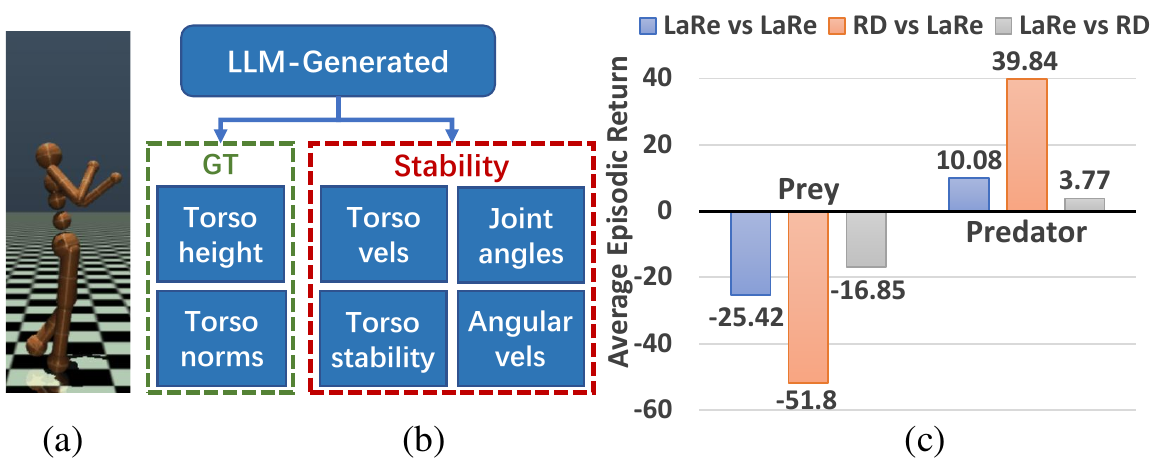}
    \caption{
    (a) The task \textit{HumanoidStandup-v4} aims to make the humanoid stand up and maintain balance.
(b) LLM-generated latent rewards additionally consider implicit factors affecting stability compared to the ground truth rewards.
    (c) Comparison between LaRe and RD on the competitive \textit{Predator-Prey (6 agents)} task. `X vs Y' represents the condition where X controls preys and Y controls predators. LaRe outperforms RD when directly pitted against it.}

    \label{fig:hetero}
\end{figure}


\begin{table}[t]
  \centering
  
  \resizebox{\linewidth}{!}{
    \begin{tabular}{c|cc|cc}
    \toprule
     \multirow{2}{*}{Tasks}& \multicolumn{2}{c|}{\texttt{corr} \small(\texttt{dims})} & \multicolumn{2}{c}{\texttt{exe\_rate}}\\
     & \small States & \small Latent Rewards & \small w/o PV&\small w/ PV\\
    \midrule
    \textit{CN~(6 agents)} & 0.02 \scriptsize(26) & \textbf{0.50 \scriptsize(5.6)} & 0\%& \textbf{100\%}\\
    \textit{
PP~(6 agents)} & 0.01 \scriptsize(28) & \textbf{0.12 \scriptsize(5.4)} & 20\%& \textbf{100\%} \\
    \textit{
HalfCheetah-v4} & 0.22 \scriptsize(17) & \textbf{0.53 \scriptsize(4.8)} &40\%& \textbf{100\%} \\
    \textit{
HumanoidStandup-v4} & \hspace{0.35em}0.20 \scriptsize(376) & \textbf{0.49 \scriptsize(5.6)} &40\%& \textbf{100\%}\\
    \bottomrule
    \end{tabular}%
    }
    \caption{\texttt{corr} denotes the average Pearson correlation coefficient.
    \texttt{dims} represents the average number of dimensions of original states or latent rewards.
    Additionally, we record the average execution rate \texttt{exe\_rate} of LLM-generated latent reward functions without pre-verification~(\small w/o PV).
    }
  \label{tab:corr}%
\end{table}%

\textbf{Reduced Reward-irrelevant Redundancy.} 
We calculate the Pearson correlation coefficient~\cite{cohen2009pearson} between each dimension of original states or LLM-generated latent rewards and ground truth dense rewards.
As shown in Table~\ref{tab:corr}, latent rewards are tighter correlated with ground truth rewards across tasks.  
Meanwhile, latent rewards' dimensions are significantly fewer than those of original states. 
The results confirm that the latent reward reduces the reward-irrelevant redundancy with task-related priors, improving reward prediction, as shown in Appendix F.5.

\textbf{Algorithm Agnostic. }
Notably, latent rewards for estimating proxy rewards are transferable to various RL backbones. 
This property ensures LaRe's application prospects, opening up possibilities to combine with real-world approaches.
We conduct detailed experiments in Appendix F.3.

\textbf{Compatible with Heterogeneous Agents.}
Latent rewards can help re-assign credits among heterogeneous agents, even in competitive scenarios. 
Like \citet{lowe2017multi}, we jointly train policies for competitive predators and preys in task \textit{Predator-Prey}.
We have the policies trained by LaRe and RD respectively compete, with preys and predators controlled by different ones.
As shown in Fig.~\ref{fig:hetero}(c), LaRe learns superior policies for both predators and preys compared to RD, suggesting enhanced credit assignment in competitive multi-agent scenarios.
This advantage can be attributed to the multifaceted nature of the latent reward.


\subsection{Ablation Studies}

\begin{figure}[t]
  \centering
    \includegraphics[width=\linewidth]{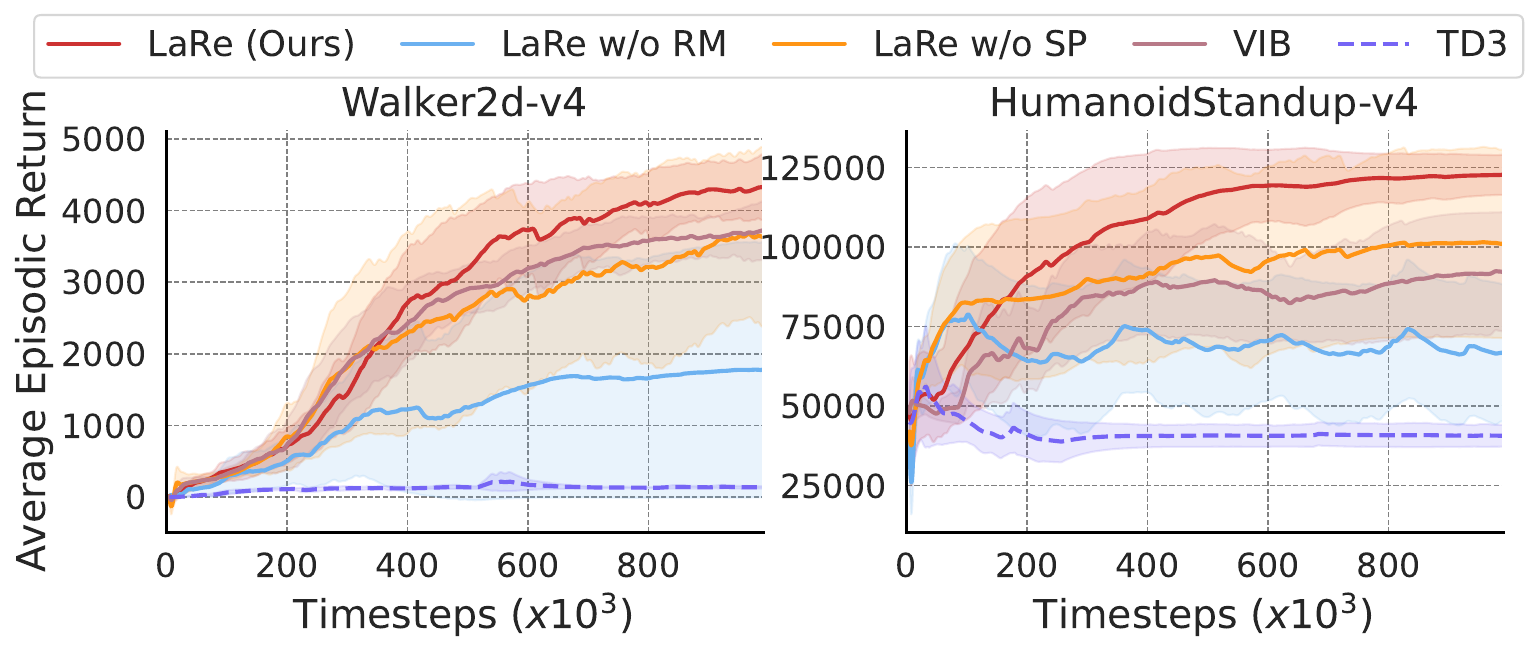}
    \caption{Ablation studies of the reward model and the proposed self-verified LLM generation, as well as comparisons of LaRe with the variational information bottleneck.}
    \label{fig:llmablation}
\end{figure}

\textbf{Reward Attributes in Latent Reward.}
To distinguish latent rewards from mere state representation,
we conduct an ablation study by removing the reward decoder model, termed ``\textbf{LaRe w/o RM}", which estimates proxy rewards by summarizing latent rewards with a sign: $\hat{r}_{sign} = \sum_{i=1}^d \text{sign}(z_r^i) \cdot z_r^i$. The signs are obtained by minimizing the estimation loss between episodic rewards and the sum of proxy rewards.
As shown in Figure~\ref{fig:llmablation}, this significant simplification outperforms the baseline with episodic rewards~(TD3), confirming that the latent reward possesses genuine reward attributes rather than just representing states.

\textbf{Self-Verification.}
We propose Self-prompting (SP) and pre-verification (PV) to reduce randomness and hallucinations in LLM inference. 
The superior performance shown in Fig.~\ref{fig:llmablation} indicates that SP effectively reduces randomness in LLM inference, resulting in improved LLM responses and RL performance.
In Table.~\ref{tab:corr}, we compare the execution rate of LLM-generated latent reward encoding functions with and without pre-verification, highlighting the significance of pre-verification.
The results validate the efficacy of our designs for integrating LLMs into RL tasks. 
Appendix F.1 examines the impact of the number of candidate responses.

\textbf{Variational Information Bottleneck.}
We propose an alternative method to utilize the Variational Information Bottleneck~(VIB)~\cite{alemi2016deep} to derive latent rewards. 
However, as shown in Fig.~\ref{fig:llmablation} and Appendix F.2, LaRe significantly outperforms VIB, which underscores the importance of LLMs serving as generalized encoders of environmental information to integrate task-related priors effectively.
We compare LaRe with LLM reward design in Appendix F.4.
\section{Conclusion}

In this work, we present LaRe, a LLM-empowered framework for credit assignment in episodic reinforcement learning with task-related prior.
The framework is centered on the latent reward, whose dimensions reflect distinct aspects of task performance evaluation.
We utilize LLM's coding abilities to address the linguistic-symbolic misalignment in integrating LLM into RL tasks and propose self-verification to ensure reliable LLM inference.
This work
(i) addresses previously overlooked research questions, including training difficulties caused by redundancy and the multifaceted nature of mission performance evaluation;
(ii) develops a practical algorithm that achieves superior performance both theoretically and empirically;
(iii) advances the integration of LLM prior knowledge into RL through semantically interpretable multifaceted performance evaluation.

\textbf{Limitations \& Future Works: }
Our work focuses on tasks with symbolic states using LLM, while future research might broaden the application to images-based tasks by employing advanced multi-modal LLMs.
As LLM capabilities advance, prompt design for LaRe will become easier, requiring less task-specific information and further reducing the manual workload.
This work lays the foundation for enhancing LLM-powered credit assignment in RL, with promising potential for complex decision-making scenarios.
\section{Acknowledgments}
This work was supported by the National Key R\&D Program of China under Grant 2018AAA0102801.

\bibliography{aaai25}
\newpage

\onecolumn
\appendix

\section{A. LLM Prompts and Responses}\label{app:prompts}

Below are the prompt template, example task information, and LLM response in our work. For further details, please refer to the code. Our prompt design incorporates the chain-of-thought technique~\cite{wei2022chain}.

\begin{tcolorbox}[title = {Prompt Template}]
\textbf{ROLE INSTRUCTION:} \\
You are good at understanding tasks and writing python codes.\\
You should fully understand the provided task and describe the exact observation and action form in the current map. \\
Then, based on your understanding and the goal of the task, analyze potential positive and negative behaviours or statuses that can be reflected in the observation and action.\\
Finally, write an evaluation function that returns factors evaluating the current status from different aspects. \\

Note:\\
1. Do not use information you are not given! \\
2. Focus on the most relevant evaluation factors and use information in observation as little as possible. \\
3. The code should be generic, complete and not contain omissions! \\
4. Avoid dividing by zero!\\
5. The input variable is in the form of (batch\_size, dim), please return a list of several evaluation factor arrays, each in the form of (batch\_size, 1). \\

Please think step by step and adhere to the following JSON format (just replace the () with your answer):\\
\{\\
\hbox{\ \ \ \ }Understand: (your thought about the task),\\
\hbox{\ \ \ \ }Analyze: (think step by step and analyze potential positive and negative behaviors or statuses that can be reflected in which part of the observation and action), \\
\hbox{\ \ \ \ }Functions: (a python function with the form of `def evaluation\_func(observation, action): ... return [a list of evaluation factor arrays]')\\
\}
\\
\\
\textbf{SELF-PROMPTING:} \\

You have generated several evaluation functions. 
Please summarize them and generate a new evaluation function that incorporates all the evaluation factors.
If there are other important evaluation factors, please include them as well.

\end{tcolorbox}

\begin{tcolorbox}[title = {Task Information of \textit{HumanoidStandup-v4}}]

\textbf{TASK DESCRIPTION:}\\
The 3D bipedal robot is designed to simulate a human. It has a torso (abdomen) with a pair of legs and arms. 
The legs each consist of two links, and so the arms (representing the knees and elbows respectively). 
The environment starts with the humanoid laying on the ground. 
The goal of the environment is to make the humanoid stand up and then keep it standing by applying torques on the various hinges. 
Maintain safe control and prevent excessive torque norm. \\

\textbf{STATE-ACTION FORM:}\\
The observation is 376 dimensions, the first 45 of which are about all the position and velocity: \\
\hbox{\ \ \ \ }0: position: z-coordinate of the torso (centre); \\
\hbox{\ \ \ \ }1: angle: x-orientation of the torso (centre); \\
\hbox{\ \ \ \ }2: angle: y-orientation of the torso (centre); \\
\hbox{\ \ \ \ }3: angle: z-orientation of the torso (centre); \\
\hbox{\ \ \ \ }4: angle: w-orientation of the torso (centre); \\
\hbox{\ \ \ \ }5: angle: z-angle of the abdomen (in lower\_waist); \\
\hbox{\ \ \ \ }6: angle: y-angle of the abdomen (in lower\_waist); \\
\hbox{\ \ \ \ }7: angle: x-angle of the abdomen (in pelvis); \\
\hbox{\ \ \ \ }8: angle: x-coordinate of angle between pelvis and right hip (in right\_thigh); \\
\hbox{\ \ \ \ }9: angle: z-coordinate of angle between pelvis and right hip (in right\_thigh); \\
\hbox{\ \ \ \ }10: angle: y-coordinate of angle between pelvis and right hip (in right\_thigh);\\
\end{tcolorbox}
\begin{tcolorbox}[title = {Task Information of \textit{HumanoidStandup-v4}}]
\hbox{\ \ \ \ }$\quad$11: angle: angle between right hip and the right shin (in right\_knee); \\
\hbox{\ \ \ \ }12: angle: x-coordinate of angle between pelvis and left hip (in left\_thigh); \\
\hbox{\ \ \ \ }13: angle: z-coordinate of angle between pelvis and left hip (in left\_thigh); \\
\hbox{\ \ \ \ }14: angle: y-coordinate of angle between pelvis and left hip (in left\_thigh); \\
\hbox{\ \ \ \ }15: angle: angle between left hip and the left shin (in left\_knee); \\
\hbox{\ \ \ \ }16: angle: coordinate-1 (multi-axis) angle between torso and right arm (in right\_upper\_arm); \\
\hbox{\ \ \ \ }17: angle: coordinate-2 (multi-axis) angle between torso and right arm (in right\_upper\_arm); \\
\hbox{\ \ \ \ }18: angle: angle between right upper arm and right\_lower\_arm; \\
\hbox{\ \ \ \ }19: angle: coordinate-1 (multi-axis) angle between torso and left arm (in left\_upper\_arm); \\
\hbox{\ \ \ \ }20: angle: coordinate-2 (multi-axis) angle between torso and left arm (in left\_upper\_arm); \\
\hbox{\ \ \ \ }\ldots \textit{(Omit for brevity. Please refer to the code for the complete state form.)}\\

The action is 17 dimensions. An action represents the torques applied at the hinge joints.
\end{tcolorbox}

\begin{tcolorbox}[title = {An example of LLM-generated responses in \textit{HumanoidStandup-v4}}]

\{\\
\hbox{\ \ \ \ }\textbf{Understand}: \\
\hbox{\ \ \ \ }\hbox{\ \ \ \ }The task involves controlling a 3D bipedal robot to make it stand up from a laying position and maintain its standing by applying torques on various hinges. The observation space is 376-dimensional, with the first 45 dimensions giving essential positional and velocity information of various parts of the robot. The goal is to evaluate the robot\'s current status using relevant factors derived from this observation.\\
\hbox{\ \ \ \ }\textbf{Analyze}: \\
\hbox{\ \ \ \ }\hbox{\ \ \ \ }To evaluate the robot's status, we need to consider the following key aspects from the observation and action: \\
\hbox{\ \ \ \ }\hbox{\ \ \ \ }1. The elevation of the torso (to see if the robot is standing). \\
\hbox{\ \ \ \ }\hbox{\ \ \ \ }2. The angles and angular velocities of the torso and limbs (to check for stability). \\
\hbox{\ \ \ \ }\hbox{\ \ \ \ }3. The direction and magnitude of torques (to ensure control and safety). \\
\hbox{\ \ \ \ }\hbox{\ \ \ \ }Positive behaviors include the torso being at a higher position (indicating standing), angular velocities being small (indicating stability), and torques being within safe limits. \\
\hbox{\ \ \ \ }\hbox{\ \ \ \ }Negative behaviors include the torso being low (indicating falling), high angular velocities (indicating instability), and excessive torques (indicating unsafe control).\\
\hbox{\ \ \ \ }\textbf{Functions}:\\
\hbox{\ \ \ \ }\hbox{\ \ \ \ }def evaluation\_func(observation, action):\\    \hbox{\ \ \ \ }\hbox{\ \ \ \ }\hbox{\ \ \ \ }import numpy as np\\    
\hbox{\ \ \ \ }\hbox{\ \ \ \ }\hbox{\ \ \ \ }batch\_size = observation.shape[0]\\    \hbox{\ \ \ \ }\hbox{\ \ \ \ }\hbox{\ \ \ \ }eval\_factors = []\\ 
\hbox{\ \ \ \ }\hbox{\ \ \ \ }\hbox{\ \ \ \ }\# Factor 1: Height of torso (z-coordinate)\\    \hbox{\ \ \ \ }\hbox{\ \ \ \ }\hbox{\ \ \ \ }height\_torso = observation[:, 0].reshape(batch\_size, 1)\\  
\hbox{\ \ \ \ }\hbox{\ \ \ \ }\hbox{\ \ \ \ }eval\_factors.append(height\_torso)\\   
\hbox{\ \ \ \ }\hbox{\ \ \ \ }\hbox{\ \ \ \ }\# Factor 2: Torso orientation stability (desired upright orientation)\\   
\hbox{\ \ \ \ }\hbox{\ \ \ \ }\hbox{\ \ \ \ }torso\_orientation = np.abs(observation[:, 1:5])  \# Ideal: [0, 0, 0, 1] for x, y, z, w\\  
\hbox{\ \ \ \ }\hbox{\ \ \ \ }\hbox{\ \ \ \ }orientation\_stability = 1 - np.sum(torso\_orientation, axis=1, keepdims=True)\\   
\hbox{\ \ \ \ }\hbox{\ \ \ \ }\hbox{\ \ \ \ }eval\_factors.append(orientation\_stability)\\
\hbox{\ \ \ \ }\hbox{\ \ \ \ }\hbox{\ \ \ \ }\# Factor 3: Torso linear and angular velocities (low values are better)\\    
\hbox{\ \ \ \ }\hbox{\ \ \ \ }\hbox{\ \ \ \ }linear\_velocities = np.linalg.norm(observation[:, 22:25], axis=1, keepdims=True)\\  
\hbox{\ \ \ \ }\hbox{\ \ \ \ }\hbox{\ \ \ \ }angular\_velocities = np.linalg.norm(observation[:, 25:28], axis=1, keepdims=True)\\    
\hbox{\ \ \ \ }\hbox{\ \ \ \ }\hbox{\ \ \ \ }movement\_stability = 1 / (1 + linear\_velocities + angular\_velocities)  \# To normalize\\  
\hbox{\ \ \ \ }\hbox{\ \ \ \ }\hbox{\ \ \ \ }eval\_factors.append(movement\_stability)\\
\hbox{\ \ \ \ }\hbox{\ \ \ \ }\hbox{\ \ \ \ }\# Factor 4: Joint angles in neutral/expected ranges\\    
\hbox{\ \ \ \ }\hbox{\ \ \ \ }\hbox{\ \ \ \ }joint\_angles = observation[:, 5:22]  \# Examples include abdomen, hips, knees, and arms\\  
\hbox{\ \ \ \ }\hbox{\ \ \ \ }\hbox{\ \ \ \ }neutral\_angles = np.abs(joint\_angles)  \# Assuming small values are better\\   
\hbox{\ \ \ \ }\hbox{\ \ \ \ }\hbox{\ \ \ \ }joint\_angle\_stability = 1 / (1 + np.sum(neutral\_angles, axis=1, keepdims=True))\\
\hbox{\ \ \ \ }\hbox{\ \ \ \ }\hbox{\ \ \ \ }eval\_factors.append(joint\_angle\_stability)\\
\end{tcolorbox}
\begin{tcolorbox}[title = {An example of LLM-generated responses in \textit{HumanoidStandup-v4}}]
$\qquad\quad$\#Factor 5: Angular velocities of joints (low values are better)\\    
\hbox{\ \ \ \ }\hbox{\ \ \ \ }\hbox{\ \ \ \ }joint\_angular\_velocities = np.linalg.norm(observation[:, 28:45], axis=1, keepdims=True)\\    
\hbox{\ \ \ \ }\hbox{\ \ \ \ }\hbox{\ \ \ \ }joint\_velocity\_stability = 1 / (1 + joint\_angular\_velocities)  \# Normalize\\    \hbox{\ \ \ \ }\hbox{\ \ \ \ }\hbox{\ \ \ \ }eval\_factors.append(joint\_velocity\_stability)\\ 
\hbox{\ \ \ \ }\hbox{\ \ \ \ }\hbox{\ \ \ \ }\# Factor 6: Torque norms (low values are better)\\    
\hbox{\ \ \ \ }\hbox{\ \ \ \ }\hbox{\ \ \ \ }torque\_norms = np.linalg.norm(action, axis=1, keepdims=True)\\    
\hbox{\ \ \ \ }\hbox{\ \ \ \ }\hbox{\ \ \ \ }torque\_efficiency = 1 / (1 + torque\_norms)\\  
\hbox{\ \ \ \ }\hbox{\ \ \ \ }\hbox{\ \ \ \ }eval\_factors.append(torque\_efficiency)\\
\hbox{\ \ \ \ }\hbox{\ \ \ \ }return eval\_factors"\\
\}

\end{tcolorbox}
\section{B. Proof}\label{app:proof}

\textbf{Notations and Definitions: }
Following ~\citet{efroni2021reinforcement}, we introduce additional notations and definitions:
$K\in\mathbb{N}$ denotes the total number of episodes and $k\in\{1,...,K\}$ denotes an index of an episode.
$T\in\mathbb{N}$ denotes the episode length and $t\in\{1,...,T\}$ denotes a timestep in an episode.
$\hat{\nu}^\phi_k\in\mathbb{R}^{T\mathcal{\norm{D}}}$ is the empirical latent reward visitation vector  given by $\hat{\nu}^\phi_{k}(z_r,t) = \mathbb{I}(z_r=z_{r,t}^k) \in [0,1]$.
Let $\Bar{r}\in\mathbb{R}^{T\mathcal{\norm{D}}}$ denote the noisy version of the true reward function $r$ in the latent reward space.
Then, the episodic reward of $k$-th episode can be represented as $R_k=\Bar{r}^\top\hat{\nu}^{\phi}_k$.
Additionally, we define the empirical latent reward frequency vector $\hat{h}^\phi_k\in \R^{\mathcal{\norm{D}}}$ where $\hat{h}^\phi_k(z_r)=\sum_{t=1}^T \hat{\nu}^\phi_{k}(z_r,t) \in [0,T]$. Finally, for any positive definite matrix $M\in\R^{m\times m}$ and any vector $x\in\R^m$, we define $\norm*{x}_M=\sqrt{x^\top Mx}$.

We estimate the reward by a regularized least-squares estimator, i.e., for some $\lambda>0$,
$$\hat{r}^\phi_k\in \arg\min_r \br*{\sum_{l=1}^k (\inner{\hat{h}^\phi_l,r} - R_l)^2 + \lambda I_{\mathcal{\norm{D}}}},$$
which has a closed form solution
$$\hat{r}^\phi_k = ((H^{\phi}_k)^\top H^{\phi}_k +\lambda I_{\mathcal{\norm{D}}})^{-1}Y^\phi_k \eqdef (A^\phi_{k})^{-1} Y^\phi_k
$$
where $H_k^\phi\in \mathbb{R}^{k\times \mathcal{\norm{D}}}$ is a matrix with $\{(\hat{h}_l^\phi)^\top\}_{l=1}^k$ in its rows. $Y^\phi_k =\sum_{l=1}^k\hat{h}^\phi_l R_l\in \R^{\mathcal{\norm{D}}}$ and $A^\phi_{k}= (H_k^\phi)^\top H_k^\phi +\lambda I_{\mathcal{\norm{D}}}\in\R^{\mathcal{\norm{D}}\times \mathcal{\norm{D}}}$.

\begin{theorem}[\citet{abbasi2011improved}, Theorem 2]

Let $\brc*{F_k}_{k=0}^\infty$ be a filtration. Let $\brc*{\eta_k}_{k=0}^\infty$ be a real-valued stochastic process such that $\eta_k$ is $F_k$-measurable and $\eta_k$ is conditionally $\delta$-sub-Gaussian for $\delta\geq 0$. Let $\brc*{x_k}_{k=0}^\infty$ be an $\R^m$-valued stochastic process s.t. $X_k$ is $F_{k-1}$-measurable and $\norm{x_k}\leq L$. Define $y_k = \inner{x_k,w}+\eta_t$ and assume that $\norm{w}_2\leq R $ and $\lambda>0$. Let
\begin{align*}
    \hat{w}_t = (X_k^TX_k+\lambda I_d)^{-1} X_k^T Y_k,
\end{align*}
where $X_k$ is the matrix whose rows are $x_1^T,..,x_t^T$ and $Y_k = (y_1,..,y_k)^T$. Then, for any $\delta>0$ with probability at least $1-\delta$ for all, $t\geq 0$ $w$ lies in the set
\begin{align*}
    \brc*{w\in \R^m: \norm{\hat{w}_k - w}_{V_k} \leq \delta\sqrt{m\log{\frac{1+kL^2/\lambda}{\delta}}} + \lambda^{1/2}R}.
\end{align*}
\end{theorem}

The theorem above can be extended to our setting, as stated in Proposition~\ref{proposition: concentration of reward}:

\rewardConcentration*
\begin{proof}
We define the stochastic process $\eta_k = \sum_{t=1}^T \br*{\Bar{r}(z_{r,t}^k)-r(z_{r,t}^k)} = \sum_{t=1}^T\Bar{r}(z_{r,t}^k) - r^\top\hat{h}^\phi_k$ and the filtration $\tilde{F}_k=\sigma\br*{\hat{h}^\phi_1,\dots,\hat{h}^\phi_{k+1},\eta_1\dots,\eta_k}$. Notably, $\hat{h}^\phi_k \in \mathbb{R}^d$ is $F_{k-1}$ measurable, $\eta_k$ is $F_k$ measurable and that $\eta_k$ is $\sqrt{T/4}$ sub-Gaussian given $F_{k-1}$, as a (centered) sum of $T$ conditionally independent random variables bounded in $[0,1]$.
Obviously, $\norm{\hat{h}^\phi_k}_2\le\norm{\hat{h}^\phi_k}_1=T$. Following \cite{efroni2021reinforcement}, we assume that $\norm{r}_2\le\sqrt{\mathcal{\norm{D}}}$.

Then, applying Theorem 2 of \citet{abbasi2011improved} mentioned above, for any $\frac{\delta}{10}>0$, with probability at least $1-\frac{\delta}{10}$, we derive that
\begin{align*}
    \forall k\ge0,\, \norm{r-\hat{r}^\phi_k}_{A^\phi_{k}} \leq  \sqrt{\frac{1}{4}T \mathcal{\norm{D}}\log\br*{\frac{1+kT^2/\lambda}{\delta/10}}}+\sqrt{\lambda \mathcal{\norm{D}}} \eqdef l^\phi_k. 
\end{align*}

Compared to Proposition 1 in \citet{efroni2021reinforcement}, since $\mathcal{\norm{D}}<\mathcal{\norm{S}}\mathcal{\norm{A}}$, it is evident that $l^\phi_k<l_k=\sqrt{\frac{1}{4}\mathcal{\norm{S}}\mathcal{\norm{A}}T\log\br*{\frac{1+kT^2/\lambda}{\delta/10}}}+\sqrt{\lambda \mathcal{\norm{S}}\mathcal{\norm{A}}}$, i.e., a tighter concentration bound of reward.
\end{proof}

\begin{restatable}[\citet{efroni2021reinforcement}, Lemma 8]{lemma-rst}{ExpectedPotentialLemma} \label{lemma: oldlemma8}
    Let $\brc*{F_{k}^s}_{k=1}^\infty$ be a filtration such that for any $k$ $F_{k}\subseteq F_{k}^s$. Assume that $h_{\pi_k} = \Ex{\hat{h}_k \mid F_{k-1}^s}$. Then, for all $\lambda>0$, it holds that
    \begin{align*}
        &\sum_{k=0}^{K}\Ex{ \norm{h_{\pi_k}}_{A_{k-1}^{-1}}  | F_{k-1}} \leq 4\sqrt{\frac{T^2}{\lambda}K\log\br*{\frac{2K}{\delta}}} + \sqrt{2\frac{T^2}{\lambda}K\mathcal{\norm{S}}\mathcal{\norm{A}}\log\br*{\lambda + \frac{KT^2}{\mathcal{\norm{S}}\mathcal{\norm{A}}}}}\enspace,
    \end{align*}
    uniformly for all $K>0$, with probability greater than $1-\delta.$
\end{restatable} 

We define the average occupation measurement of policy $\pi$ in the the MDP's latent reward space as $h_{\pi_k}^\phi=\mathbb{E}\left[\hat{h}_k^\phi|F^s_{k-1}\right]\in\mathcal{D}$.
Given Lemma \ref{lemma: oldlemma8}, we can easily derive that with probability greater than $1-\delta/2$, it holds that

\begin{align}
    &\sum_{k=0}^{K}\Ex{ \norm{h^\phi_{\pi_k}}_{(A^\phi_{k-1})^{-1}}  | F_{k-1}} \leq 4\sqrt{\frac{T^2}{\lambda}K\log\br*{\frac{4K}{\delta}}} + \sqrt{2\frac{T^2}{\lambda}K\mathcal{\norm{D}}\log\br*{\lambda + \frac{KT^2}{\mathcal{\norm{D}}}}}\enspace,
    \label{eq7}
\end{align}

Based on the above conclusions, we analyze the advantage of reward-irrelevant redundancy elimination in the latent reward on the regret bound of the OFUL algorithm as used before~\cite{abbasi2011improved, efroni2021reinforcement}, with the exception that the optimization objective becomes $\max_\pi\left(\left(h_\pi^\phi\right)^\top\hat{r}^\phi_{k-1}+l_{k-1}^\phi\norm{h_\pi^\phi}_{(A^\phi_{k-1})^{-1}}\right)$.

\OFULrlPer*

\begin{proof}
We define $\mathbb{G}$ as the good event that Proposition 1 and Eq.~\eqref{eq7} both hold. Consequently, it is evident that $\Pr\{\mathbb{G}\}\ge1-\frac{\delta}{10}-\frac{\delta}{2}\ge1-\delta$.

Let $\mathbb{C}_k^\phi\eqdef \{\norm{\tilde{r}-\hat{r}^\phi_k}_{A_k^\phi}\le l_k^\phi\}$. Conditioning on $\mathbb{G}$, it holds that $r\in \mathcal{C}^\phi_k$ for all $k>0$. Thus, 
\begin{align}
    (h_{\pi_k}^\phi)^\top \hat{r}^\phi_{k-1} + l^\phi_{k-1}\norm{h^\phi_{\pi_k}}_{(A^\phi_{k-1})^{-1}} = \max_\pi\left(\left(h_\pi^\phi\right)^\top\hat{r}^\phi_{k-1}+l_{k-1}^\phi\norm{h_\pi^\phi}_{(A^\phi_{k-1})^{-1}}\right) = \max_{\pi}\max_{\tilde r\in \mathcal{C}^\phi_{k-1}} (h_\pi^\phi)^\top \tilde r \geq (h^\phi_{\pi^*})^\top r, \label{eq: optimism}
\end{align}
i.e., the algorithm is optimistic. The value function of policy $\pi$ in the MDP can be defined as $V^{\phi,\pi}=r^\top h^\phi_\pi\in \mathbb{R}$. Then, the regret, which measures the performance of the agent, can be defined as
\begin{equation}
    \rho^\phi(K)\eqdef \sum_{k=1}^K(V^*-V^{\phi,\pi_k})=\sum_{k=1}^K(r^\top h^\phi_{\pi^*}-r^\top h^\phi_{\pi_k}).
\end{equation}

Following similar analysis to~\citep{abbasi2011improved, efroni2021reinforcement}, we can bound the regret as follows
\begin{align}
    &\rho^\phi(K) = \sum_{k=1}^K(r^\top h^\phi_{\pi^*}-r^\top h^\phi_{\pi_k}) \nonumber \\
    & \leq \sum_{k=1}^K (h_{\pi_k}^\phi)^\top \hat{r}^\phi_{k-1} + l^\phi_{k-1}\norm{h^\phi_{\pi_k}}_{(A^\phi_{k-1})^{-1}} -r^\top h^\phi_{\pi_k}) \tag{Eq.~\eqref{eq: optimism}} \nonumber\\ 
    &=\sum_{k=1}^K (h_{\pi_k}^\phi)^\top (\hat{r}^\phi_{k-1} -r)+ l^\phi_{k-1}\norm{h^\phi_{\pi_k}}_{(A^\phi_{k-1})^{-1}} \nonumber\\
    &\leq \sum_{k=1}^K \norm{h_{\pi_k}^\phi}_{(A^\phi_{k-1})^{-1}}\norm{ \hat{r}^\phi_{k-1} -r}_{(A^\phi_{k-1})^{-1}} + l^\phi_{k-1}\norm{h^\phi_{\pi_k}}_{(A^\phi_{k-1})^{-1}} \nonumber\\
    &\leq 2l^\phi_{K}\sum_{k=1}^K \norm{h^\phi_{\pi_k}}_{(A^\phi_{k-1})^{-1}}, \label{eq: regret bound}
\end{align}
where the last relation holds conditioning on $\mathbb{G}$ and that $l^\phi_{K}\geq l^\phi_{k}$ for all $k\leq K$.

Following \citep{efroni2021reinforcement}, we set $\lambda = T$ and observe that conditioning on $\mathbb{G}$ it holds that 
\begin{align*}
    \sum_{k=0}^{K}\Ex{ \norm{h^\phi_{\pi_k}}_{(A^\phi_{k-1})^{-1}}  | F_{k-1}} \leq \Ocal\br*{\sqrt{TK\mathcal{\norm{D}}\log\br*{\frac{TK}{\delta}}}},
\end{align*}

Combining with Proposition 1, we conclude that
\begin{equation}
    \rho^\phi(K)\leq \Ocal\br*{T \mathcal{\norm{D}} \sqrt{ K} \log\br*{\frac{KT}{\delta}}} < \Ocal\br*{T \mathcal{\norm{S}}\mathcal{\norm{A}} \sqrt{ K} \log\br*{\frac{KT}{\delta}}}.
\end{equation}

\end{proof}

\section{C. Experimental Details}
\label{app:expdetail}


\subsection{Baselines}
\begin{itemize}
    \item \textbf{RD}: learns a reward model to decompose the episodic reward into individual time steps of the trajectory~\cite{arjona2019rudder}, using the least-squares-based implementation~\cite{efroni2021reinforcement}.
    \item \textbf{IRCR}~\cite{gangwani2020learning}: involves smoothing in the trajectory space and proposes uniform reward redistribution, assigning equal contribution to each state-action pair in a trajectory.
    \item \textbf{RRD}~\cite{ren2021learning}: introduces a surrogate loss to enhance scalability through randomized Monte Carlo return estimation, bridging RD and IRCR.
    \item \textbf{RRD\_unbiased}~\cite{ren2021learning}: is a variant of RRD that provides an unbiased estimation of RD.
    \item \textbf{Diaster}~\cite{lin2024episodic}: decomposes episodic rewards into contributions assigned to two subtrajectories divided at any cut point.
    \item \textbf{AREL}~\cite{xiao2022agent}: focuses on multi-agent settings and employs attention mechanisms to capture influence across both temporal and agent dimensions.
    \item \textbf{STAS}~\cite{chen2023stas}: also utilizes attention mechanisms while leveraging the Shapley Value to redistribute each agent's payoffs.
\end{itemize}

\subsection{Implementation details and Hyperparameters}
We implemented the algorithms in MuJoCo using the official TD3 codes~\footnote{\url{https://github.com/sfujim/TD3}} and those in MPE using the author-provided STAS codes~\footnote{\url{https://github.com/zowiezhang/STAS}}, maintaining the respective hyperparameters.

The implementations and hyperparameters of RRD, RRD\_unbiased, and IRCR were sourced from the official RRD codes~\footnote{\url{https://github.com/Stilwell-Git/Randomized-Return-Decomposition}}, except that we set the subsequence length of RRD to 10 in MPE. 
For AREL, we directly adopted the reward model from the official codes~\footnote{\url{https://github.com/baicenxiao/AREL}}.
We repreduced Diaster using TD3 based on the official codebase~\footnote{\url{https://github.com/HxLyn3/Diaster}}.

Regarding LaRe, the key hyperparameter is the number of candidate responses $n$, which is set to 5 by default. We investigate the impact of varied $n$ in Appendix F.
The task information is extracted from the documentation of Gymnasium~\footnote{\url{https://gymnasium.farama.org/environments/mujoco}} for MuJoCo and PettingZoo~\cite{terry2021pettingzoo}~\footnote{\url{https://pettingzoo.farama.org/environments/mpe/}} for MPE.

\subsection{Tasks}
In this section, we briefly describe the tasks evaluated in this work.

\subsubsection{MuJoCo}

We use the MuJoCo benchmark provided by Gymnasium~\cite{towers_gymnasium_2023} and evaluate our method on four representative tasks—\textit{Reacher-v4}, \textit{Walker2d-v4}, \textit{HalfCheetah-v4}, \textit{HumanoidStandup-v4}—each with different state space dimensions. 
The maximum episode length is the same for all tasks, set at 1000 steps.
These tasks are illustrated in Figure~\ref{fig:mujocotask}.

\begin{figure}[htbp]
\centering
\includegraphics[width=0.8\linewidth]{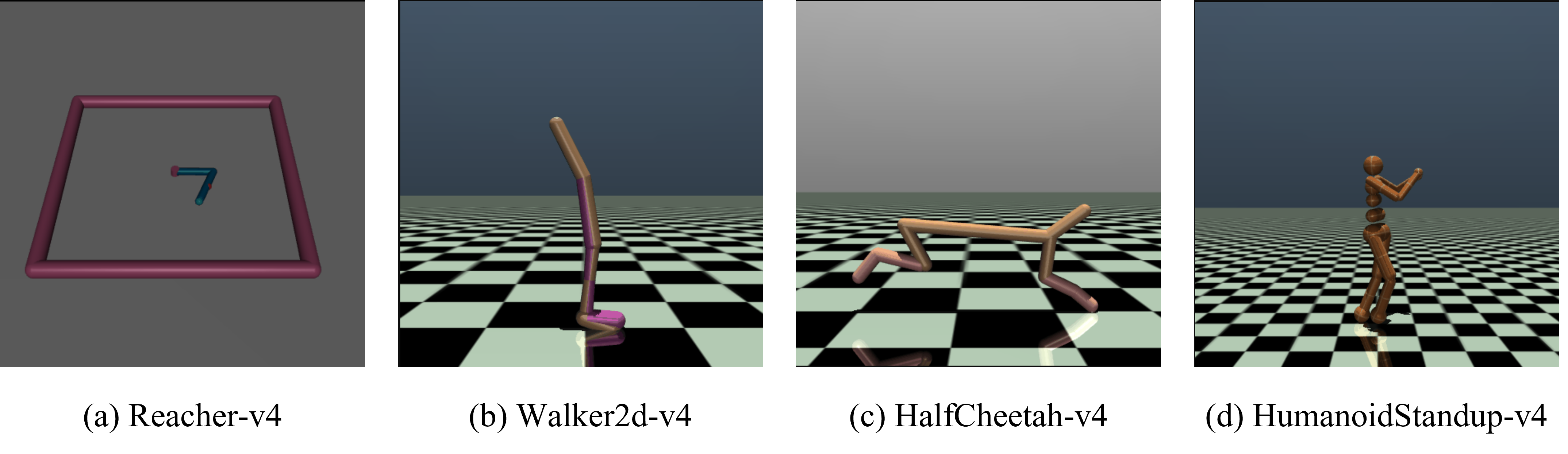}
\caption{The illustrations of tasks in MuJoCo.}
\label{fig:mujocotask}
\end{figure}

\begin{itemize}
    \item \textit{Reacher-v4}: The goal is to move a two-jointed robot arm close to a target that is spawned at a random position.
    \item \textit{Walker2d-v4}: The walker is a two-dimensional two-legged figure. The goal is to walk in the in the forward direction by applying torques on the six hinges connecting the seven body parts.
    \item \textit{HalfCheetah-v4}: The HalfCheetah is a 2-dimensional robot. The goal is to apply a torque on the joints to make the cheetah run forward as fast as possible. 
    \item \textit{HumanoidStandup-v4}: The 3D bipedal robot is designed to emulate human movement. The environment begins with the humanoid lying on the ground, and the objective is for the humanoid to stand up and maintain balance.
\end{itemize}

The details of these tasks, including observation space and action space, are listed in Table.~\ref{apptab:mujoco_tasks}.

\begin{table}[htbp]
\centering

\begin{tabular}{ccccc}
\toprule
\textbf{MuJoCo tasks} &  \textbf{observation shape} & \textbf{action shape (continuous)} \\ \midrule
\textit{Reacher-v4 }              & 11                   & 2                      \\ 
\textit{Walker2d-v4}        & 17                   & 6                            \\ 
\textit{HalfCheetah-v4 }          & 17                   & 6                   \\ 
\textit{HumanoidStandup-v4}         & 376                  & 17         
\\ \bottomrule
\end{tabular}
\caption{Details of MuJoCo tasks}

\label{apptab:mujoco_tasks}
\end{table}

\subsubsection{Multiple-Particle Environment (MPE)}

For MPE, we use two cooperative scenarios—\textit{Cooperative-Navigation~(CN)} and \textit{Predator-Prey~(PP)}—and evaluate algorithms on six tasks involving varying numbers of agents~(6, 15, 30), based on the implementations of \citet{chen2023stas}. 
These tasks are illustrated in Figure~\ref{fig:mpetask}.

\begin{itemize}
    \item \textit{Cooperative-Navigation~(CN)}: There are N agents and N landmarks. The agents must learn to cover all the landmarks while avoiding collisions.
    \item \textit{Predator-Prey~(PP)}: There are N predators and M preys. The preys are faster. Agents should control the predators to catch the prey, which are controlled by a pretrained policy. Additionally, obstacles are present on the map to impede the movement of both predators and preys.
\end{itemize}

Notably, instead of using shared rewards as in \citet{chen2023stas}, we introduce minor modifications to generate individual rewards for each agent at each transition step. 
Additionally, we evaluate algorithms in the competitive \textit{Predator-Prey} task, where policies for both predators and prey are trained jointly in an episodic setting, analogous to the setting used by \citet{lowe2017multi}.

\begin{figure}[htbp]
\centering
\includegraphics[width=0.5\linewidth]{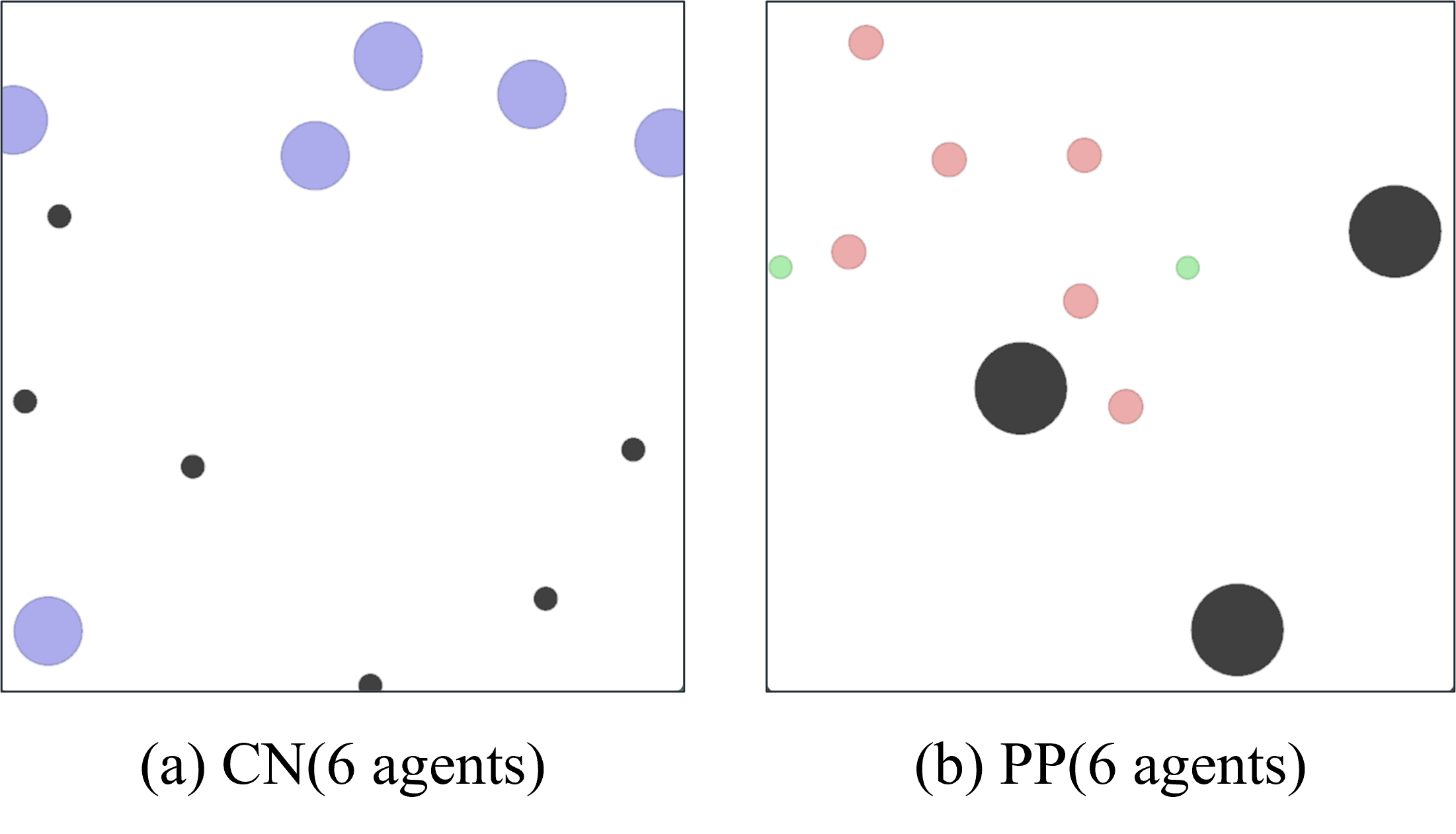}
\caption{
The illustrations of tasks in MPE. In \textit{Cooperative Navigation~(CN)}, the agents are depicted as black dots and the navigation targets as purple dots. In \textit{Predator-Prey~(PP)}, the predators are shown as pink dots, the preys as green dots, and the obstacles as black dots.}
\label{fig:mpetask}
\end{figure}

\begin{table}[htbp]
\centering
\begin{tabular}{ccccc}
\toprule
\textbf{MPE tasks} & \textbf{n\_agents}& \textbf{observation shape}  & \textbf{action num (discrete)}  \\ \midrule
\textit{CN (6 agents) }   & 6    & 26  & 5 \\ 
\textit{CN (15 agents) } & 15   & 26  & 5 \\ 
\textit{CN (30 agents) } & 30   & 26  & 5 \\ 
\textit{PP (6 agents) }   & 6    & 28  & 5 \\ 
\textit{PP (15 agents) } & 15   & 34  & 5 \\ 
\textit{PP (30 agents) } & 30   & 46  & 5 \\ 
\bottomrule
\end{tabular}

\caption{Details of MPE tasks}
\label{apptab:mpe_tasks}
\end{table}

\subsection{Resources}

We utilize a server running Ubuntu 20.04.3, equipped with 8 NVIDIA RTX 3090 GPUs and 2 AMD 7H12 CPUs, to conduct all experiments.

\section{D. A Brand New Task: \textit{Triangle Area}}

To exclude the probability that LaRe's success relies on LLM's familiarity with the chosen tasks, we've designed a brand new task based on MPE, termed \textit{Triangle Area}, which LLM has never encountered before. The task is detailed as follows:

As shown in Fig.~\ref{fig:taskarea}, there are three agents and three obstacles. The agents must cooperate to maximize the area of the triangle they form while avoiding collisions with the obstacles. 
The observation space consists of 14 dimensions, including the positions of the agents and obstacles, as well as the velocities of the agents. 
The action space for each agent is discrete, comprising five movement actions.

As shown in Fig.~\ref{fig:areacurve}, LaRe achieves superior performance compared to the baselines and is comparible with IPPO trained with dense rewards, , consistent with the main results.
The result confirms that it is the LLM's generalization capabilities that empower LaRe's effectiveness in deriving efficacy latent rewards in diverse new tasks.

\begin{figure}[htbp]
\centering
\begin{subfigure}[b]{0.27\linewidth}
    \centering
    \includegraphics[width=\linewidth]{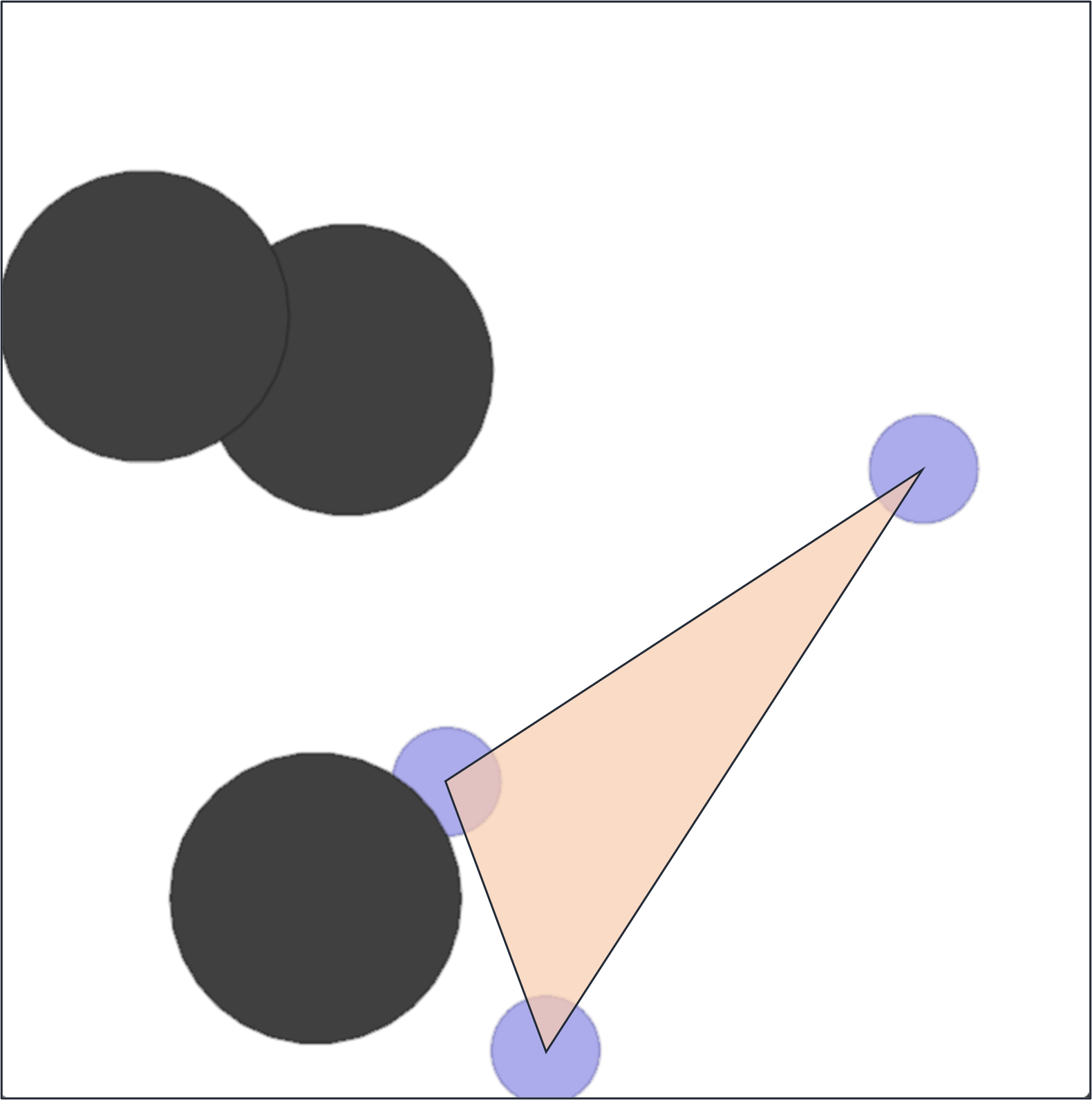}
    \captionsetup{font=footnotesize}
	\caption{\ \ }
    \label{fig:taskarea}
\end{subfigure}
\hspace{5pt}
\begin{subfigure}[b]{0.32\linewidth}
    \centering
    \includegraphics[width=\linewidth]{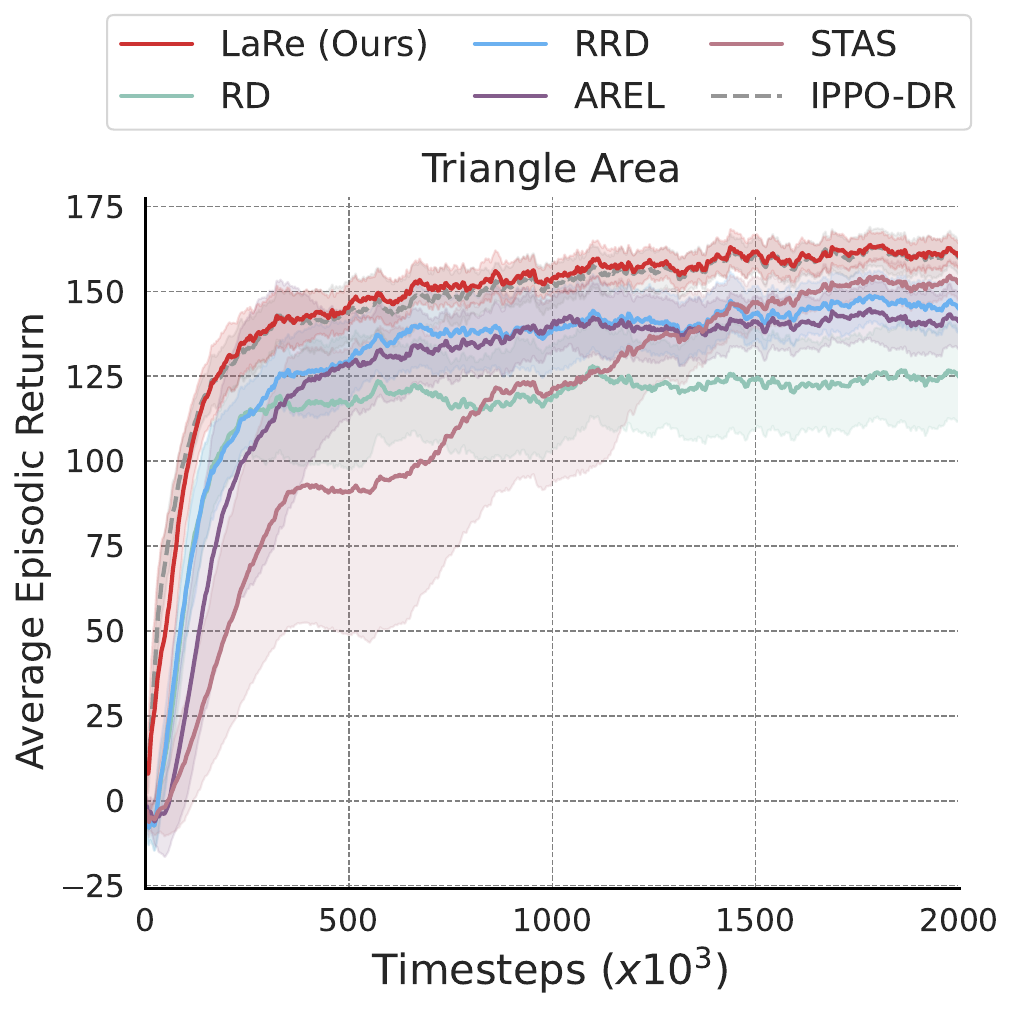}

    \captionsetup{font=footnotesize}
	\caption{\ \ }
    \label{fig:areacurve}
\end{subfigure}

\caption{
(a) The illustration depicts a newly designed task, \textit{Triangle Area}, where agents (purple dots) and obstacles (black dots) are shown. The agents must cooperate to expand the area of the encircled triangle (orange).
    (b) Average episode returns of different algorithms on the \textit{Triangle Area} task.
}
\label{fig:area}
\end{figure}

\begin{figure}[htbp]
    \centering
    \includegraphics[width=\linewidth]{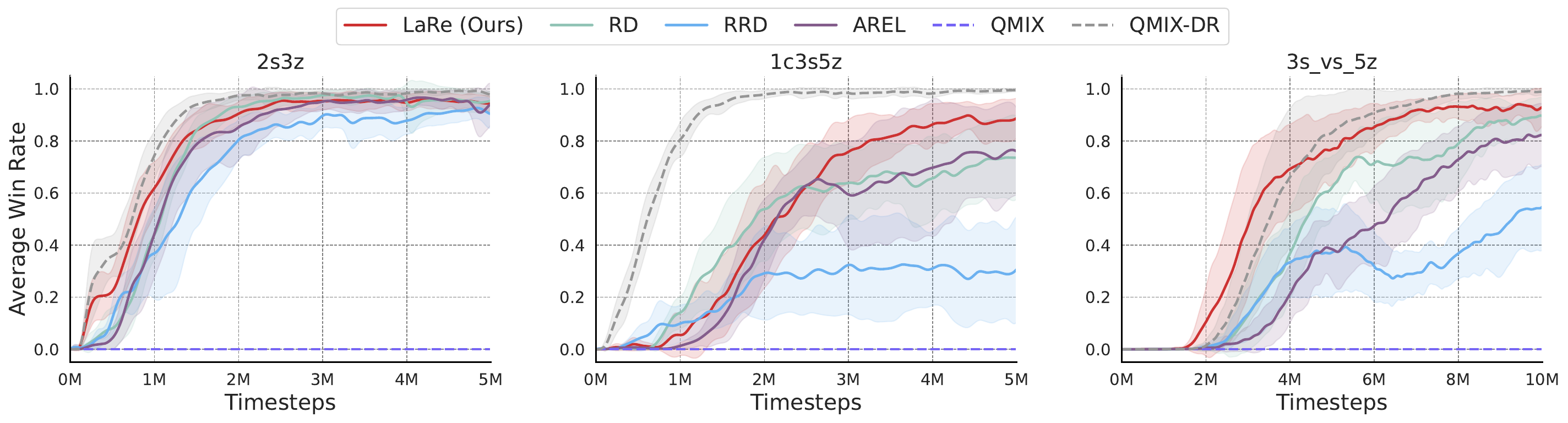}
    \caption{Average win rate of LaRe and baseline algorithms on three maps $\mathtt{2s3z, 1c3s5z}$ and $\mathtt{3s\_vs\_5z}$ in SMAC. LaRe still outperforms all baselines in this complex multi-agent environment.
    }
    \label{fig:smac}
    
\end{figure}

\begin{figure}[htbp]
    \centering
    \includegraphics[width=0.65\linewidth]{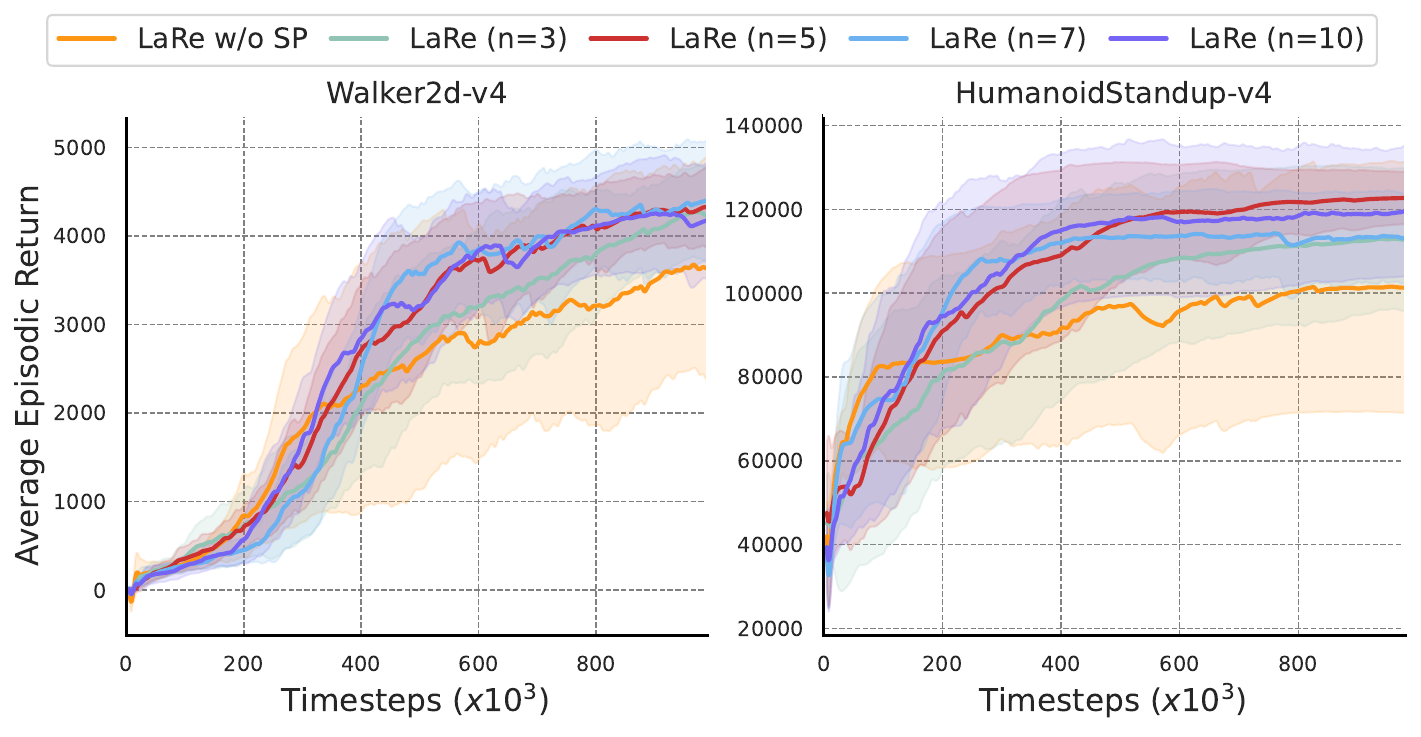}
    \caption{Average episodic return of LaRe with varying number of candidate responses $n$ on the tasks \textit{Walker2d-v4} and \textit{HumanoidStandup-v4}. 
    }
    \label{fig:llmn}
\end{figure}

\section{E. More Complicated Multi-agent tasks: SMAC}
We also conducted experiments on StarCraft Multi-agent Challenge~(SMAC)~\cite{samvelyan19smac}, a commonly-used benchmark in previous works~\cite{shao2024counterfactual, shao2023complementary},  to demonstrate that LaRe can be applied to more complex multi-agent tasks with the latent reward maintaining its effectiveness.
SMAC is a challenging task with high-dimensional state spaces and significant exploration requirements, making it difficult for independent learning methods such as IPPO~\cite{yu2022surprising} to perform well.
Therefore, we select \textbf{QMIX}~\cite{rashid2018qmix} as the base reinforcement learning algorithm, while the code and hyperparameters are adopted from PyMARL2~\footnote{\url{https://github.com/hijkzzz/pymarl2}}, which is a widely used code implementation of QMIX.
The implementations of all reward decomposition methods remains consistent as mentioned in Appendix C., except that RRD requires adjusting the hyperparameter of subsequence length for different maps.
Additionally, LaRe is also implemented based on baseline RD.
For fairness and compatibility with QMIX, all methods take the global state as input and estimate the step-wise reward as a common reward for all agents.
During the MARL training process, each algorithm is carried out with 5 random seeds.

We evaluate each algorithm on three maps: $\mathtt{2s3z, 1c3s5z}$ and $\mathtt{3s\_vs\_5z}$~(hard).
Similar to the experiments in MPE, we also include \textbf{QMIX-DR} training with ground truth dense rewards.
As shown in Fig.~\ref{fig:smac}, the baseline QMIX can hardly learn the policies due to the delayed feedback in such complicated multi-agent tasks, emphasizing the necessity of credit assignment on time scale.
Our proposed LaRe surpasses all baseline algorithms on three maps, highlighting its effectiveness in temporal credit assignment even in the complex environment.

In fact, LaRe shows a greater advantage over baselines on the relatively challenging map $\mathtt{3s\_vs\_5z}$.
Difficult maps require a more precise focus on the key factors for problem-solving, showcasing how the latent reward with incorporation of prior knowledge improves accuracy in capturing multifaceted performance factors.
This reduction of redundant information by the LLM-empowered latent reward enhances reinforcement learning process.
The performance gap between all reward decomposition methods and QMIX-DR with ground truth reward on the $\mathtt{1c3s5z}$ map is likely due to the higher heterogeneity among agents.
Though LaRe offers the multifaceted evaluation, the QMIX-based implementation of LaRe is not allowed to assign rewards for each agent, limiting its ability to inform different agents of their own contributions.
This underscores the importance of contribution allocation among agents again, as highlighted in our MPE experiments analysis.
With advancements in independent learning algorithms, combining LaRe might yield improved results in such heterogeneous tasks.

\section{F. Additional Ablation Studies}

\subsection{F.1 The Number of Candidate Responses}
To investigate the impact of the number of candidate responses, a key hyperparameter in LaRe, Fig.~\ref{fig:llmn} presents the average episodic return of LaRe with varying $n$ on the tasks \textit{Walker2d-v4} and \textit{HumanoidStandup-v4}. 
The results indicate that LaRe is not sensitive to this hyperparameter as long as a sufficient number of candidates is provided. 
This finding supports our claim that self-prompting can help reduce randomness in LLM inference.
Consequently, we set $n=5$ for all tasks in this work.

\subsection{F.2 Variational Information Bottleneck.}
We propose an alternative method for utilizing the Variational Information Bottleneck~(VIB)~\cite{alemi2016deep} to derive latent rewards. 
Similar to LaRe, we implement VIB using RD in MPE and RRD\_unbiased in MuJoCo. 
Unlike the baselines, VIB includes an encoder model with reduced output dimensions before the reward model, and minimizes the Kullback–Leibler divergence as detailed by \citet{alemi2016deep}, in addition to the reward modeling loss. 
As illustrated in Fig.\ref{fig:vib}, LaRe significantly outperforms VIB. This result highlights the critical role of LLMs as generalized encoders of environmental information in effectively integrating task-related priors.

\begin{figure}[ht]
  \centering
    \includegraphics[width=\linewidth]{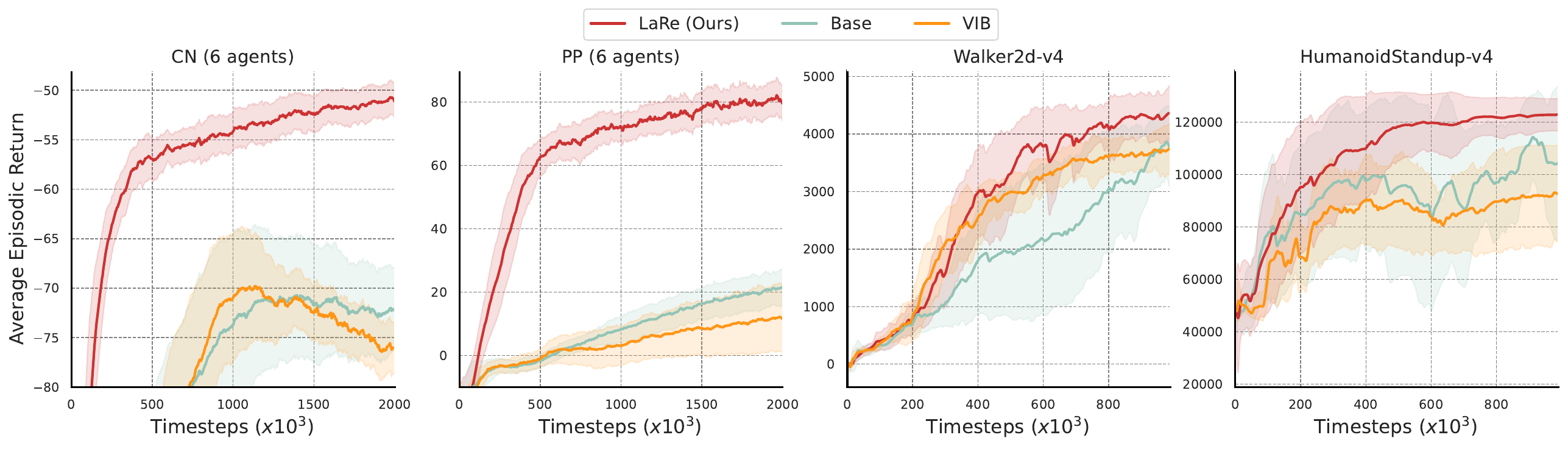}
    \caption{
    The Variational Information Bottleneck shows poor performance in extracting latent rewards, underscoring the necessity of LLMs. Note that \textbf{Base} refers to the base algorithms on which VIB is implemented, specifically RD in MPE and RRD\_unbiased in MuJoCo.}
    \label{fig:vib}
\end{figure}

\subsection{F.3 Compatible with Various RL Algorithms}

As shown in Fig.~\ref{fig:rlagnoppo}, we summarize the experimental results of combining LaRe with various RL algorithms, including TD3, SAC, DDPG, and PPO. 
The results demonstrate that our method consistently outperforms the baselines and is comparable with baselines training with dense rewards. 
This property ensures the application prospects of LaRe when combined with other real-world approaches.

\begin{figure}[htbp]
    \centering
    \includegraphics[width=\linewidth]{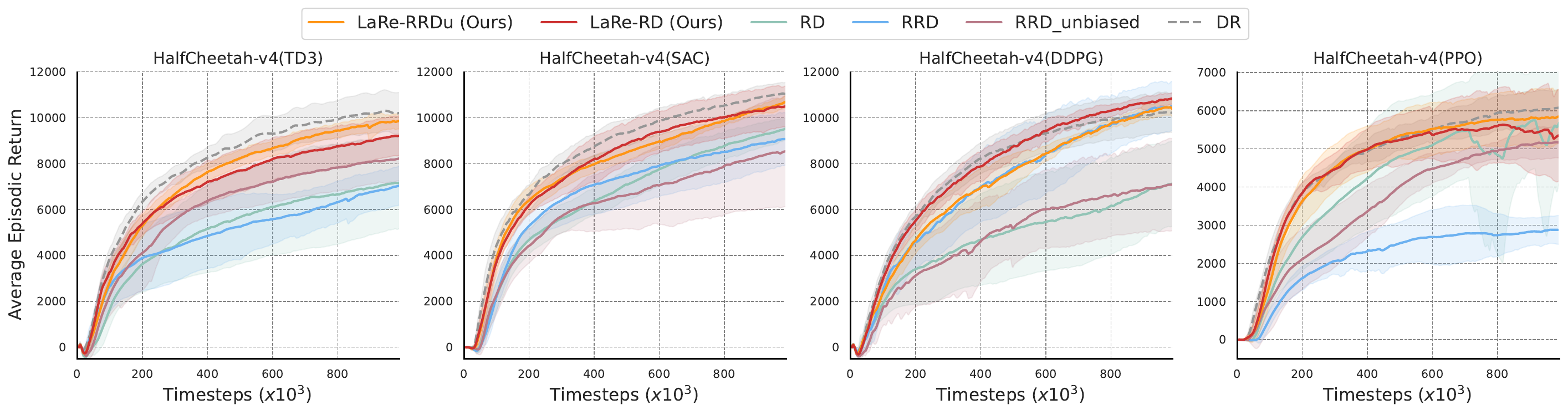}
    \caption{The complete results of combining LaRe with various RL algorithms. Note that \textbf{DR} represents training vanilla RL algorithms with dense rewards. 
    }
    \label{fig:rlagnoppo}
\end{figure}

\subsection{F.4 LLM Reward Design.}
We compare LaRe with  ``\textbf{Reward Design}", inspired by Eureka~\cite{ma2023eureka}, where a LLM designs reward functions. 
This can be seen as using LLM to define relationships between the reward and the latent reward rather than employing reward modeling. 
As shown in Figure~\ref{fig:llmrewarddesign}, LaRe consistently outperforms LLM reward design, highlighting the efficacy of incorporating latent rewards into return decomposition for optimization.
This contrasts with relying solely on LLM responses, which may establish incorrect relationships between reward factors.

\begin{figure}[htbp]
    \centering
    \includegraphics[width=0.6\linewidth]{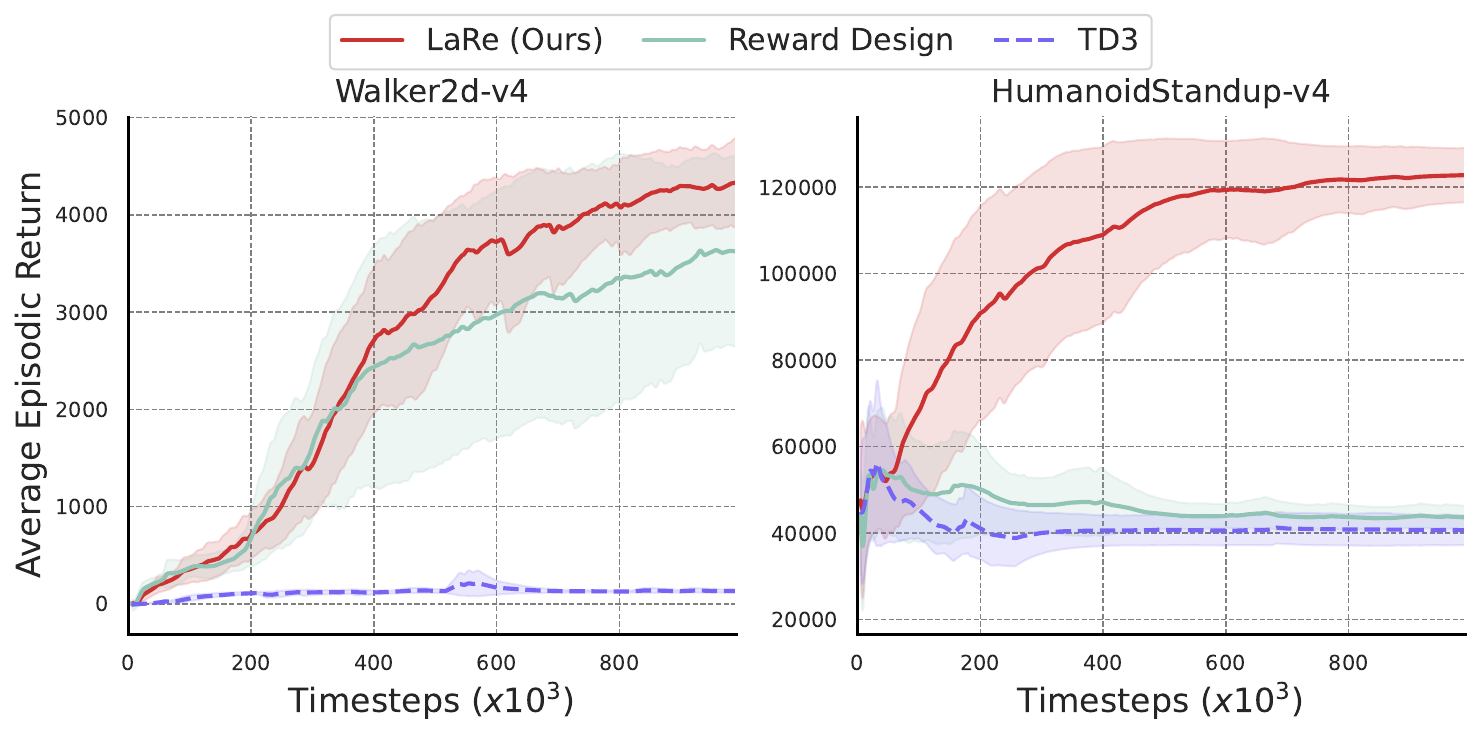}
    \caption{The comparison of LaRe with LLM reward design.
    }
    \label{fig:llmrewarddesign}
\end{figure}

\subsection{F.5 Reward Prediction Errors.}
As shown in Table.~\ref{tab:reward_pred}, we empirically evaluate the absolute difference $|\hat{r}-r|$ between proxy and ground truth rewards to show LaRe's significant improvement in attribution accuracy, as a first step in addressing ambiguous attribution with multifaceted LLM evaluation.

\begin{table}[h]
\centering
\caption{Reward Prediction Errors.}
\label{tab:reward_pred}
\begin{tabular}{lcccc}
\hline
\textbf{Task} & \textbf{LaRe} & \textbf{RRD\_unbiased} & \textbf{RD} \\ \hline
Reacher       & \textbf{0.0002 $\pm$ 0.0002} & $0.011 \pm 0.002$ & $0.004 \pm 0.003$ \\
HalfCheetah   & \textbf{0.7 $\pm$ 0.5} & $10.7 \pm 4.1$ & $9.6 \pm 8.5$ \\ \hline
\end{tabular}
\end{table}

\subsection{F.6 Extensions for Fully Sparse Reward Tasks.}
Although not aligning with the sum-form decomposition assumption, we evaluated the effectiveness of both our method and previous approaches in fully sparse reward settings, where a binary signal (indicating task completion) is provided only at the end of a trajectory. Specifically, we modified the training reward function for the \textit{HalfCheetah-v4} task to $I(final\_travel\_distance > 500)$ and used $final\_travel\_distance$ as the evaluation metric. As shown in Fig.~\ref{fig:fullysparse}, we compared LaRe against baseline methods and demonstrated that LaRe outperforms them. This superior performance may be attributed to the stronger correlation between latent rewards and task completion.

\begin{figure}[htbp]
    \centering
    \includegraphics[width=0.4\linewidth]{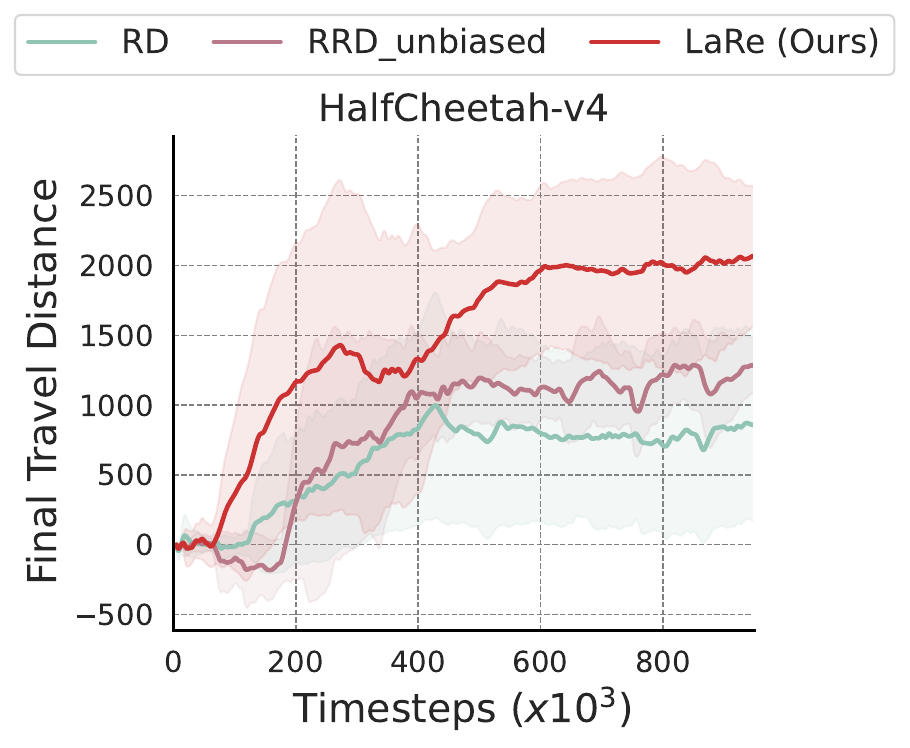}
    \caption{The comparison of LaRe against baselines on \textit{HalfCheetah-v4} with fully sparse rewards.
    }
    \label{fig:fullysparse}
\end{figure}

\section{G. Visualization}
As shown in Fig.\ref{fig:tsne}, we visualize latent rewards and original states after reducing dimensionality to 2D using t-SNE for the \textit{CN (6 agents)} task. The results illustrate that latent rewards are more strongly correlated with the ground truth rewards. This can be attributed to the multifaceted nature, which makes latent rewards better align with the ground truth rewards, suggesting easier model training.

\begin{figure}
    \centering
    \includegraphics[width=0.75\linewidth]{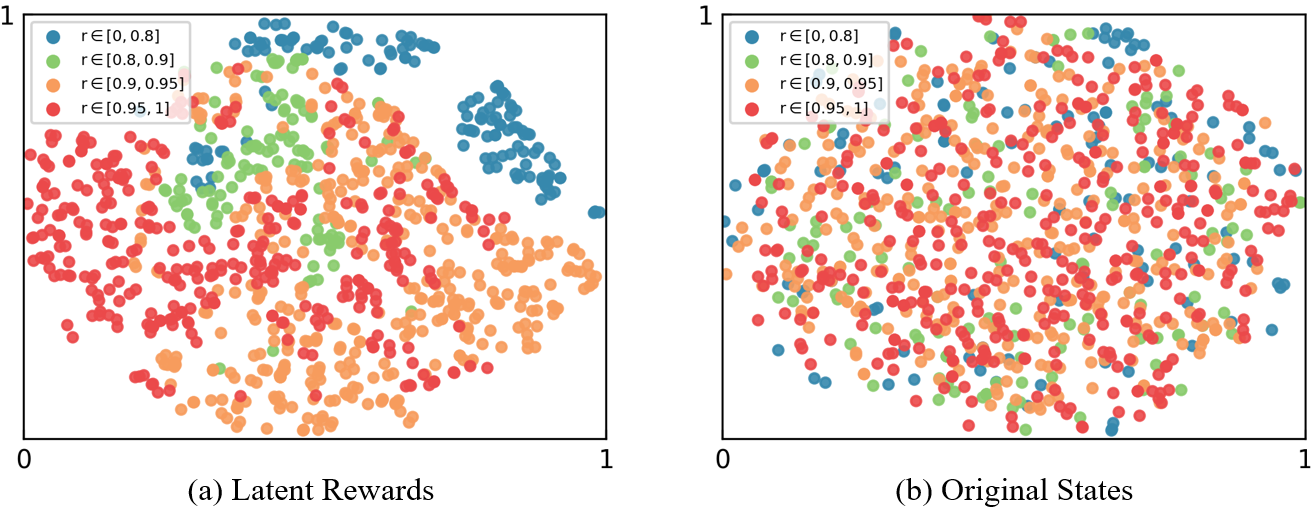}
    \caption{
    Visualization of latent rewards and original states after 2D dimensionality reduction using t-SNE for the \textit{CN (6 agents)} task.  The graph employs color coding to represent the reward values for each point.}
    \label{fig:tsne}
\end{figure}

\section{H. Discussions}

\subsection{Discussions on Ambiguous Attributions.}

By `ambiguous attributions', we mean: in methods directly fitting final rewards, reward-irrelevant features tend to lead to biased reward attributions for certain state-action pairs.

A toy example: the task is to `find a hidden door button' with a light button reachable in fewer steps. The state space includes the door and light (reward irrelevant) status. Early in learning, turning on the light is easier than opening the door, resulting more frequent `light on, door opened' trajectories. Previous methods tend to reward `light on', causing the policy to consider the light's status and get stuck in suboptimal solution. Using LLM to filter out irrelevant features (e.g., light status), LaRe enables accurate reward predictions.

Empirically, as shown in Table.~\ref{tab:reward_pred}, LaRe shows significant improvement in attribution accuracy, as a first step in addressing ambiguous attribution with multifaceted LLM evaluation.

\subsection{Future Works}

In this work, we focus primarily on tasks involving semantically meaningful symbolic states. We aim to address this by using a Vision-Language Model~(VLM)~\citep{liu2024visual} to extract and encode object information from images as latent rewards.
Additionally, this study focuses on online settings, but LaRe can be easily extended to offline episodic RL settings~\citep{shao2024counterfactual, mao2023supported, mao2024offline, mao2024doubly, mao2024supported, zhang2023sample}, representing an intriguing avenue for future research.

\end{document}